%% file: main_neurips.tex
\documentclass{article}

    \PassOptionsToPackage{square,numbers,comma}{natbib}

\usepackage[final]{neurips_2022}

\usepackage[utf8]{inputenc} 
\usepackage[T1]{fontenc}    
\usepackage{hyperref}       
\usepackage{url}            
\usepackage{booktabs}       
\usepackage{amsfonts}       
\usepackage{nicefrac}       
\usepackage{microtype}      
\usepackage{xcolor}         

\def\tdotoggle{0}
\usepackage{gene-neurips}
\usepackage{etoolbox}
\newtoggle{arxiv}
\newtoggle{neurips_cr}
\settoggle{arxiv}{false}
\newcommand{\version}[2]{\iftoggle{arxiv}{#1}{#2}}

\newcommand{\mdp}{\mathrm{MDP}}
\newcommand{\dkl}{D_{\mathrm{KL}}}
\newcommand{\vecc}[1]{\mathsf{vec}\inparen{#1}}
\newcommand{\pai}{\partial_i}

\title{Exponential Family Model-Based Reinforcement Learning via Score Matching}

\author{%
  Gene Li\\
  Toyota Technological\\
  Institute at Chicago\\
  \texttt{gene@ttic.edu}
  \And
  Junbo Li\\
  UC Santa Cruz\\
  \texttt{jli753@ucsc.edu}\\
  \And
  Anmol Kabra\\
  Toyota Technological\\
  Institute at Chicago\\
  \texttt{anmol@ttic.edu}\\
  \And
  Nathan Srebro\\
  Toyota Technological\\
  Institute at Chicago\\
  \texttt{nati@ttic.edu}\\
  \And
  Zhaoran Wang\\
  Northwestern University\\
  \texttt{zhaoranwang@gmail.com}\\
  \And
  Zhuoran Yang\\
  Yale University\\
  \texttt{zhuoran.yang@yale.edu}
}

\begin{document}

\maketitle

\input{abstract}
\input{introduction}

\input{preliminaries}

\input{model-estimation}
\input{main-results}
\input{smvsmle}
\version{\input{related-work}}{}
\input{conclusion}

\begin{ack}
This work is supported by funding from the Institute for Data, Econometrics, Algorithms, and Learning (IDEAL). We thank Pritish Kamath, Danica J. Sutherland, Akshay Krishnamurthy, and Wen Sun for helpful discussions. Part of this work was done while GL, ZW, and ZY were participating in the Simons Program on the Theoretical Foundations of Reinforcement Learning in Fall 2020.
\end{ack}

\newpage
\bibliographystyle{plainnat}
\bibliography{main_neurips.bbl}

\newpage
\appendix
\version{}{\input{related-work}}
\input{apdx-sm}
\input{apdx-regret}
\input{apdx-comparison}
\input{apdx-technical}
\include{apdx-experiments}

\end{document}

%% file: abstract.tex
\begin{abstract}
    We propose an optimistic model-based algorithm, dubbed SMRL, for finite-horizon episodic reinforcement learning (RL) when the transition model is specified by exponential family distributions with $d$ parameters and the reward is bounded and known. SMRL uses score matching, an unnormalized density estimation technique that enables efficient estimation of the model parameter by ridge regression. Under standard regularity assumptions, SMRL achieves $\tilde O(d\sqrt{H^3T})$ online regret, where $H$ is the length of each episode and $T$ is the total number of interactions (ignoring polynomial dependence on structural scale parameters). 

\end{abstract}

%% file: introduction.tex
\section{Introduction}

This paper studies the regret minimization problem for finite horizon, episodic reinforcement learning (RL) with infinitely large state and action spaces. Empirically, RL has achieved success in diverse domains, even when the problem size (measured in the number of states and actions) explodes \cite{mnih2013playing,silver2017mastering, kober2013reinforcement}. The key to developing sample-efficient algorithms is to leverage \emph{function approximation}, enabling us to generalize across different state-action pairs. Much theoretical progress has been made towards understanding function approximation in RL. Existing theory typically requires strong linearity assumptions on transition dynamics \citep[\eg~][]{yang2019reinforcement,jin2020provably,cai2020provably,modi2020sample} or action-value functions \citep[\eg~][]{lattimore2020learning, zanette2020learning} of the Markov Decision Process (MDP). However, most real world problems are \emph{nonlinear}. Our theoretical understanding of these settings remains limited. Thus, we ask the question:
\begin{center}
    \version{}{\vspace{-0.5em}}
    {\em Can we design provably efficient RL algorithms in nonlinear environments?}
    \version{}{\vspace{-0.5em}}
\end{center}
Recently, \citet{chowdhury21} introduced a nonlinear setting where the state-transition measures are finitely parameterized exponential family models, and they proposed to estimate model parameters via maximum likelihood estimation (MLE). The exponential family is a well-studied and powerful statistical framework, so it is a natural model class to consider beyond linear models. \citeauthor{chowdhury21} study exponential family transitions of the form:
\begin{equation}\label{eq:intro-model}
    \P_{W_0}(s'|s,a) = q(s') \exp\inparen{\ip{\psi(s'), W_0\phi(s,a)} - Z_{sa}(W_0)},
\end{equation}
where $\psi \in \R^{d_\psi}$ and $\phi \in \R^{d_\phi}$ are known feature mappings, $q$ is a known base measure, $W_0$ is the unknown parameter to be learned, and $Z_{sa}$ is the log partition function that ensures the density integrates to 1. This transition model covers both linear dynamical systems as well as the nonlinear dynamical system (nonLDS), introduced by \citet{mania2020active}\version{ (see \Cref{fig:models})}{}. Linear dynamical systems with quadratic rewards, \ie the linear quadratic regulator (LQR), have received much attention recently as an important testbench for RL in unknown, complex environments \cite{fazel2018global, simchowitz2020naive,kakade2020information}. Thus, the work of \citeauthor{chowdhury21} is a crucial step in bridging the gap between RL and continuous control.

\version{
\begin{figure}
    \centering
    \begin{tikzpicture}
        \node[draw, rectangle, rounded corners, text centered, text depth = 4cm, text width = 0.5\linewidth,fill=green!10] (r3) at (0,1){Exponential Family Transitions (Def. \ref{assume:ef-model}):\\$s_{t+1} \sim \P_W(\cdot \vert s_t, a_t)$};
        \node[draw, rectangle, rounded corners, text centered, text depth = 2.5cm, text width = 0.45\linewidth,fill=blue!10] (r2) at (0,0.5){Nonlinear Dynamical Systems \cite{mania2020active}:\\$s_{t+1} = W \phi(s_t, a_t) + \gauss{0}{\sigma^2 I}$};
        \node[draw, rectangle, rounded corners, text centered, text width = 0.35\linewidth, minimum width = 2cm, minimum height = 1.5cm, fill=red!10] (r) at (0,0){Linear Dynamical Systems:\\$s_{t+1} = As_t + B a_t + \gauss{0}{\sigma^2 I}$};
    \end{tikzpicture}  
    \caption{A diagram summarizing the relationship between various transition models.}
    \label{fig:models}
\end{figure}}{}

However, MLE has several shortcomings. In order to estimate the parameter $W_0$ in (\ref{eq:intro-model}), MLE requires estimating the log partition function $Z_{sa}$, which is computationally intensive. Practical implementations for MLE which estimate the log partition function via Markov Chain Monte Carlo (MCMC) methods can be slow and induce approximation errors \cite{carreira2005contrastive}. These approximation errors can propagate in undesirable ways to the algorithm's planning procedure. Since the MLE $\hat{W}$ cannot be computed in closed form, \citeauthor{chowdhury21} leave their estimator implicitly defined as solutions of the likelihood equations. As is typical for upper confidence RL (UCRL) algorithms, one constructs high probability confidence sets around the estimator. Due to the challenging modeling assumption, \citeauthor{chowdhury21} employ confidence sets which are sums of KL divergences taken over the dataset.

In this work, we bypass these difficulties by instead proposing to learn the model parameters with \emph{score matching}, an unnormalized density estimation technique introduced by \citet{hyvarinen2005estimation}. Score matching provides an explicit, easily computable closed form estimator for the model parameters by solving a certain ridge regression problem (\Cref{thm:sm-loss}). Moreover, we can employ high probability confidence sets which are ellipsoids centered at the estimator, a standard component in prior theoretical work on linear bandits and linear MDPs \cite[\eg][]{abbasi2011improved, jin2020provably}.

Our main results are as follows:
\version{}{\vspace{-\topsep}}
\version{\begin{itemize}}{\begin{itemize}[leftmargin=0.5cm]}
    \item We extend prior work on the score matching estimator in the i.i.d.\ setting by proving nonasymptotic concentration guarantees for non-i.i.d.~data (\Cref{thm:concentration}).
    \item  We consider regret minimization for the setting of exponential family transitions and bounded and known rewards. We design a model-based algorithm, dubbed SMRL, which achieves regret of $\tilde{O}(d \sqrt{H^3 T})$, with polynomial dependence on structural scale parameters (\Cref{thm:smrl-regret}). Here, $d = d_\psi\times d_\phi$ is the total number of parameters of $W_0$, $H$ is the episode length, and $T$ is the total number of interactions. In each episode, SMRL uses score matching as a computationally efficient subroutine to estimate $W_0$ from data, then it constructs elliptic confidence regions around the estimator which contain $W_0$ w.h.p.~and chooses policies optimistically based on such confidence regions. (This work assumes computational oracle access to an optimistic planner.)
\end{itemize}
Our regret guarantee matches that of Exp-UCRL, the model-based algorithm proposed by \citeauthor{chowdhury21} When specialized to the nonLDS with bounded costs and features, score matching and MLE are equivalent estimators (\Cref{prop:equiv}). Here, the work of \citet{kakade2020information} gives a tighter guarantee of $\tilde{O}(\sqrt{d_\phi(d_\phi + d_{\psi}+ H) H^2 T})$; however we stress that our analysis applies to a broader class of models. Broadly speaking, we view score matching and MLE as complementary estimation techniques; while MLE relies on less assumptions, score matching enjoys computational efficiency and allows us to simplify both the algorithm and proofs. A detailed comparison is deferred to \Cref{sec:compare}. \version{}{In this work, we mainly compare against the papers \cite{chowdhury21,kakade2020information}, but a broader summary of related work can be found in \Cref{sec:relatedwork}.}

\version{\subsection{Definitions and Notation}}{\paragraph{Notation.}}
For a vector $x\in \R^d$, we let $\norm{x} := \norm{x}_2$ denote the $\ell_2$ norm. For a matrix $M\in \R^{n\times d}$, we denote $\vecc{M}\in \R^{nd}$ to be the vectorized version of $M$. For a matrix $M$, we also denote $\norm{M}_2$ to be the operator norm and $\norm{M}_F$ to be the Frobenius norm, \ie $\norm{M}_F := \norm{\vecc{M}}$. We also let $e_i\in \R^d$ and $E_{ij} \in \R^{n\times d}$ denote the canonical basis vectors and matrices respectively. For positive semidefinite matrices $A,B$, we let $A\preceq B$ to be $B-A \succeq 0$. For positive semidefinite matrix $A$ and vector $x$ we define $\norm{x}_{A} := \sqrt{x^\top A x}$. For any $n\in \N$, we let $[n] := \{1,2,\dots, n\}$. For a twice differentiable function $f: \R^m \mapsto \R^n$ and any $i\in[m]$, we let $\pai f(x) := \inparen{\tfrac{\partial}{\partial x_i} f_1(x), \dots, \tfrac{\partial}{\partial x_i} f_n(x) }^\top \in \R^n$ and $\pai^2 f(x) := \inparen{\tfrac{\partial^2}{\partial x_i^2} f_1(x), \dots, \tfrac{\partial^2}{\partial x_i^2} f_n(x) }^\top \in \R^n$. We use the word ``algorithm'' liberally, since methods discussed in this paper as well as other papers require solving optimization procedures which can be computationally intractable.

\raggedbottom
\nopagebreak

%% file: preliminaries.tex
\section{Problem statement}\label{sec:prelim}

We consider the setting of an episodic Markov Decision Process, denoted by $\mdp(\calS,\calA, H,\P,r)$, where 
$\calS$ is the state space, $\calA$ is the action space, $H\in\N$ is the horizon length of each episode, $\P$ is state transition probability measure, and $r:\calS\times\calA \mapsto \R$ is the reward function. 

The agent interacts with the episodic MDP as follows. At the beginning of each episode, a state $s_1$ is chosen by an adversary and revealed to the agent. The agent picks a \textbf{policy function}, which is a collection of (possibly random) functions $\pi:=\{\pi_h:\calS\mapsto \Delta(\calA)\}_{h\in[H]}$ that determines the agent's strategy for interacting with the world. For each step $h\in[H]$, the agent observes the state $s_h$ and plays action $a_h \sim \pi_h(s_h)$. Afterwards, they observe reward $r_h(s_h,a_h)$, and the MDP evolves to a new state $s_{h+1} \sim \P\left(\cdot \given s_h, a_h\right)$. The episode terminates at state $s_{H+1}$ after which the world resets. 

The goal of the agent is to maximize their cumulative rewards through interactions with the MDP. Concretely, in our model-based setting the agent knows the reward function $r$ and that the transition model $\P$ lies in some model class $\calP$, and they want to pick policies every episode to minimize \textbf{regret}, which we formally define later on.

Now we define the value function and action-value function. For every policy $\pi$, we can define a \textbf{value function} $V_{\P, h}^\pi: \calS \mapsto \R$, which is the expected value of the cumulative future rewards when the agent plays policy $\pi$ starting from state $s$ in step $h$, and the world transitions according to $\P$. In this paper, we include $\P$ in the subscript since we will analyze value functions for different models; if clear from context, we will drop the subscript $\P$. Specifically, we have:
\[V_{\P,h}^\pi(s) := \E_{\P}\insquare{\sum_{h'=h}^H r_{h'} \inparen{s_{h'}, a_{h'} } \given s_h=s, a_{h:H} \sim \pi },\quad \forall s\in\calS,h\in[H].\]
Similarly, we define the \textbf{action-value} functions $Q_{\P,h}^\pi(s,a): \calS\times \calA \mapsto \R$ to be the expected value of cumulative rewards starting from a state-action pair in step $h$, following $\pi$ afterwards:
\[Q_{\P,h}^\pi(s,a) := \E_{\P}\insquare{\sum_{h'=h}^H r_{h'} \inparen{s_{h'}, a_{h'} } \given s_h=s, a_h = a, a_{h+1:H} \sim \pi},\quad \forall (s,a)\in\calS\times \calA,h\in[H].\]
An optimal policy $\pi^\star$ is defined to be the policy such that the corresponding value function $V_{\P,h}^{\pi^\star}(s)$ is maximized at every state $s\in \calS$ and step $h\in [H]$. Without loss of generality, it suffices to consider deterministic policies \cite{szepesvari2010algorithms}. Given knowledge of the MDP $(\calS,\calA, H,\P,r)$, the optimal value function and action-value function can be computed via dynamic programming \cite{sutton2018reinforcement}; then the optimal policy can be computed as the policy that acts greedily with respect to the optional action-value function, \ie $\pi^\star_h(s) = \argmax_{a\in \calA}Q^\star_{\P, h}(s,a)$.\info{GL: added some more words about computing the optimal policy}

In the online setting, we will measure the performance of an agent interacting with the MDP over $K$ episodes via the notion of \textbf{regret}. In every episode $k\in [K]$, an adversary presents the agent with a state $s_1^k$, and the agent then chooses a policy $\pi^k$. The regret over $K$ episodes is the expected suboptimality of the agent's choice of policy $\pi^k$ compared to the optimal policy $\pi^\star$:
\[\calR(K) := \sum_{k=1}^K \inparen{V_{ 1}^{\pi^\star}(s_1^k)-V_{1}^{\pi^k}(s_1^k) }.\]
Implicit in the notation $\calR(K)$ are the adversary's choice of initial states; our results for regret will hold for any sequence of adversarially chosen $\{s_1^k\}_{k\in[K]}$.\info{GL: added this clarification on the initial states} We will also denote $T:= KH$ as the total number of interactions the agent makes with the world. 

\subsection{Exponential family transitions}

We consider the setting when the transition model class $\calP$ is given by exponential family transitions and the the reward function $r:\calS\times \calA\mapsto \R$ is bounded a.s. in $[0,1]$ and known to the learner.\version{}{\footnote{Our results extend to settings where the rewards are not known but instead lie in some class $\calR \subseteq (\calS\times \calA \to \R)$ by including an additional reward estimation procedure in our algorithm; the regret would additionally depend on the complexity of $\calR$.}}

\begin{definition}[Exponential family transitions, c.f.,~\cite{chowdhury21}]\label{assume:ef-model}
Suppose $\calS \subseteq \R^{d_s}$ and $\calA$ is any arbitrary action set. Fix feature mappings $\psi:\calS \mapsto \R^{d_{\psi}}$ and $\phi:\calS \times \calA \mapsto \R^{d_{\phi}}$, as well as base measure $q: \calS\to \R$. For any matrix $W \in \R^{d_\psi \times d_\phi}$, let:
\begin{equation}\label{eq:ef-model}
    \P_{W}(s'|s,a) := q(s') \exp\inparen{\ip{\psi(s'), W\phi(s,a)} - Z_{sa}(W)}, 
\end{equation}
where $Z_{sa}(\cdot)$ is the log-partition function, which is completely determined once $\psi$, $\phi$, $q$, and $W$ are specified. Then we define the \textbf{exponential family transitions} model class $\calP(\psi, \phi, q)$ as:
\[
    \calP(\psi, \phi, q) := \inbraces{\P_W: \int_\calS q(s') \exp(\ip{\psi(s'), W\phi(s,a)}) \ ds' < \infty, \ \forall (s,a)\in \calS\times \calA}.
\]
Since $\psi, \phi, q$ are taken to be fixed and known to the learner, we will write the model class as $\calP$.
\end{definition}

Along with this assumption, we introduce a notational convention. Given some real or vector-valued measurable function $f(s')$, we will write $\E^{W}_{sa} f(s')$ to denote the expected value of $f$ when $s'$ is drawn from the conditional distribution $\P_W(\cdot|s,a)$, i.e. $\E^{W}_{sa} f(s') := \int_{\calS} f(s') \P_W(s'|s,a) ds'$.

\version{\paragraph{Discussion.}The conditional exponential family model naturally arises in the well-known statistical framework of generalized linear models (GLMs) \cite{nelder1972generalized, wedderburn1974quasi}, which we briefly review. GLMs model situations where the expected response $Y \in \R$, is conditionally related to the covariates $X$ as $\E[Y|X] = \mu(\theta^\top \phi_X)$, where $\mu:\R \to \R$ is the so-called \emph{link function}, $\theta\in \R^d$ is some unknown parameter, and $\phi$ is a feature mapping for $X$. To define a GLM, one first defines a probability distribution which belongs to the \emph{canonical exponential family}, i.e. the probability density has the form:
\[p_\eta(y) = q(y)\cdot \exp(y\cdot \eta - Z(\eta)),\]
where $q(\cdot)$ is a real valued function, $\eta$ is a real-valued parameter, and $Z(\cdot)$ is a twice differentiable function. Then, the GLM associated with the exponential family model is determined by $\P_\theta(y|x) := p_{\theta^\top \phi_x}(y) = q(y)\exp(y\cdot \theta^\top\phi_x - Z(\theta^\top \phi_x))$; it can be shown that $\E[Y|X] = \mu(\theta^\top \phi_X) = Z'(\eta)|_{\eta = \theta^\top \phi_X}$. Therefore, \Cref{assume:ef-model} can be viewed as a multivariate generalization of the GLM, where the response is a function of the next state $\psi(s')$ and the covariates are current state-action pair $(s,a)$.}{}
\citeauthor{chowdhury21} state their results for a setting where the unknown matrix $W_0 = \sum_{i=1}^d \theta_i A_i$, where the $A_i\in \R^{d_\psi\times d_\phi}$ are known matrices and $\theta \in \R^d$ is unknown. This setting can be viewed as a nonlinear analog of the linear mixture model considered in \citep{modi2020sample,ayoub2020model}. \Cref{assume:ef-model} is a special case with $d = d_\psi\times d_\phi$ and $A_{ij}:=E_{ij}$. Our results can be extended to their general setting with minor modification. Quantitatively, we would replace factors of $d_\psi\times d_\phi$ with $d$ in both the concentration and regret guarantees, and similar to \citeauthor{chowdhury21} we would introduce constants which depend on $A_i$. For simplicity of presentation, we study the fully unknown matrix setting.


\version{The known and bounded rewards assumption is fairly standard in model-based reinforcement learning \cite[\eg][]{yang2019reinforcement, ayoub2020model,chowdhury21}. Our results easily extend to settings where the rewards are not known but instead lie in some class $\calR \subseteq (\calS\times \calA \to \R)$, by including an additional reward estimation procedure in our algorithm; the regret guarantee would additionally depend on the complexity of $\calR$, \ie via a complexity measure such as eluder dimension. For example, if the expected reward is an unknown linear function of the state-action feature vector, the regret would increase by an additive factor of $\tilde{O}(d\sqrt{T})$. The main focus of this paper is to understand the statistical and computational complexity associated with the more difficult task of estimating state transition measures; as such, we omit the extension to unknown rewards.}{}

\subsection{Relationship to (non)linear dynamical systems} 
We now describe how \Cref{assume:ef-model} generalizes the previously studied model class of (non)linear dynamical systems which have been explored in reinforcement learning and control theory literature.\info{GL: included more details about the relationship between assumption and lds/nonlds.} \version{A visual depiction is provided in \Cref{fig:models}.}{}

First, we take a step back and describe linear dynamical systems (LDS), which govern the transition dynamics of the LQR problem.\footnote{Strictly speaking, our results do not handle unbounded costs, so they do not apply to the LQR problem.} An LDS is defined by the following transition dynamics:
\[s' = As + Ba + \eps, \text{ where } \eps \sim \gauss{0}{\Sigma}.\]
where $s, s'\in \R^{d_s}$, $a\in \R^{d_a}$, $A, B$ are appropriately sized parameter matrices, and $\Sigma\in \R^{d_s \times d_s}$ is a known covariance matrix. The problem of estimating $(A,B)$, known as \emph{system identification}, has a long history (see \Cref{sec:relatedwork} for more details).

Recently, system identification and regret minimization have been studied for nonlinear generalizations of LDS \cite{mania2020active, kakade2020information}. In this paper, we refer to this setting as the \emph{nonlinear dynamical system} (or nonLDS for short).\footnote{\citet{kakade2020information} study kernelized version of this model, which they call the \emph{kernelized nonlinear regulator}.} The nonLDS is described by the state transition model:
\[s' = W_0\phi(s,a) + \eps, \text{ where } \eps \sim \gauss{0}{\Sigma}.\]
By setting $\phi(s,a) = [s,a]^\top$ and $W_0 = [A \ B]$, we recover the classical linear dynamical system. The nonLDS (and by extension the LDS) are special cases of \Cref{assume:ef-model}. \version{By writing out the probability density function of the multivariate gaussian distribution, one sees that the exponential family model under consideration is:
\begin{align*}
    &\P_{W_0}(s'|s,a) := \frac{1}{(2\pi)^{d_s/2} \det(\Sigma)^{1/2}} \cdot \exp\inparen{-\frac{1}{2} (s' - W_0 \phi(s,a))^\top  \Sigma^{-1}(s' - W_0 \phi(s,a)) } \\
    = \ &\frac{1}{(2\pi)^{d_s/2} \det(\Sigma)^{1/2}} \cdot \exp\inparen{-\frac{1}{2} \norm{s'}_{\Sigma^{-1}}^2}\cdot \exp\inparen{\ip{\Sigma^{-1}s', W_0 \phi(s,a)}- \frac{1}{2} \norm{W_0\phi(s,a)}_{\Sigma^{-1}}^2 },
\end{align*}
and from here it is easy to read off the functions:
\[q(s') = \frac{1}{(2\pi)^{d_s/2} \det(\Sigma)^{1/2}} \cdot \exp\inparen{-\frac{1}{2} \norm{s'}_{\Sigma^{-1}}^2},  \ \psi(s') = \Sigma^{-1}s', \ Z_{sa}(W_0) = \frac{1}{2} \norm{W_0\phi(s,a)}_{\Sigma^{-1}}^2.\]}{This can be seen by writing out the pdf of the multivariate Gaussian distribution to get:
\[q(s') = \tfrac{1}{(2\pi)^{d_s/2} \det(\Sigma)^{1/2}} \exp\Big(- \tfrac{\norm{s'}_{\Sigma^{-1}}^2}{2}\Big), \ \psi(s') = \Sigma^{-1}s', \ Z_{sa}(W_0) = \tfrac{\norm{W_0\phi(s,a)}_{\Sigma^{-1}}^2}{2}.\]}
Lastly, note that \Cref{assume:ef-model} is more general than that of the nonLDS, whose base measure $q(\cdot)$ and feature mapping $\psi(\cdot)$ must take a specific form given by the multivariate Gaussian. \Cref{assume:ef-model} gives extra flexibility in the functions $q$, $\psi$, and $\phi$, which can be regarded as \emph{design choices} for the practitioner. For example, one can pick the mapping $\psi$ the output of a neural network which captures the relevant features for the transition to $s'$; this is not permitted under the nonLDS setting. 

\info{removed multivariate GLMs discussion.}

%% file: model-estimation.tex
\section{Model estimation via score matching}\label{sec:sm}

In this section, we present the score matching method, the subroutine in our RL algorithm that estimates model parameters. We also introduce structural assumptions that enable us to derive a nonasymptotic concentration guarantee for the score matching estimator.

\subsection{Background on score matching}

Suppose we want to estimate the conditional density $\P(s'|s,a)$ of the form (\ref{eq:ef-model}), given a dataset $\calD =\{(s_t, a_t, s_{t+1})\}_{t\in[T]}$. MLE is the natural candidate for this estimation procedure, but it suffers from pitfalls. Solving for the MLE requires computing the log-partition function $Z_{sa}(\cdot)$. If the log-partition function is not known in closed form, it can be estimated using Markov Chain Monte Carlo methods \cite{brooks2011handbook, carreira2005contrastive, dai2019exponential}; however, this procedure may be computationally expensive. For some settings such as kernelized exponential families, MLE fails due to ill-posedness \cite{fukumizu2009exponential,sriperumbudur2017density}.

\version{In our application, it is important to derive closed-form solution of an estimator with high probability confidence intervals in order to apply the upper confidence bound (UCB) technique; in general, we cannot achieve this with the MLE estimator.\footnote{Due to the gaussian density, the nonLDS is an special case when it is actually possible to compute closed form solutions with confidence intervals via MLE, as shown in \cite{kakade2020information}.} As we will see, via score matching we get an easy-to-compute estimator and high probability confidence sets which are ellipsoids.}{}

\citet{hyvarinen2005estimation, hyvarinen2007some} proposed score matching as an alternative to minimizing the log likelihood. Score matching minimizes the Fischer divergence, which is the expected squared distance between the score functions $\nabla_{s'} \log \P_{W}(s'|s,a)$. Specifically, we define the divergence between $\P_{W_0}$ and $\P_{W}$ for fixed $(s,a)$ as:
\begin{equation}
    J(\P_{W_0}(\cdot | s,a) \Vert \P_{W}(\cdot| s,a)) := \frac{1}{2} \int_{\calS} \P_{W_0}(s'|s,a) \norm{ \nabla_{s'} \log \tfrac{\P_{W_0}(s'|s,a) }{\P_{W}(s'|s,a)} }^2 ds'. 
\end{equation}
Before proceeding with the exposition of the score matching estimator, we list standard regularity conditions that are required for the analysis of score matching \citep[\cf][]{sriperumbudur2017density,arbel2018kernel}.\info{GL: moved conditions to main body}

\begin{assumption}[Regularity conditions]\label{assume:regularity}~\
\version{}{\vspace{-\topsep}}
\begin{enumerate}[label={\bf (\Alph*)}]
    \item $\calS$ is a non-empty open subset of $\R^{d_s}$ with piecewise smooth boundary $\partial \calS := \bar{\calS} - \calS$, where $\bar{\calS}$ is the closure of $\calS$.
    \item (Differentiability): $\psi(\cdot)$ is twice continuously differentiable on $\calS$ with respect to each coordinate $i\in[d_s]$, and $\partial^j_i\psi(s)$ is continuously extensible to $\bar{\calS}$ for all $j\in \{1,2\}, i\in [d_s]$.
    \item (Boundary Condition): For all $(s,a)\in \calS\times \calA$ and $i\in [d_s]$, as $s'\to \partial \calS$, we have: \[\norm{\partial_i \psi(s')} \P_{W_0}(s'|s,a) = o(\norm{s'}^{1-d_s}).\]
    \item (Integrability): For all $i\in [d_s]$, $(s,a)\in \calS \times \calA$, let $p_{sa} := \P_{W_0}(\cdot|s,a)$. Then:
\begin{align*}
    \norm{\pai \psi(s')} \in L^{2}(\calS, p_{sa}), \ \norm{\pai^2 \psi(s')} &\in L^{1}(\calS, p_{sa}), \ \norm{\pai \psi(s')} \pai \log q(s') \in L^{1}(\calS, p_{sa}).
\end{align*}
\end{enumerate}
\end{assumption}

The key insight of \citeauthor{hyvarinen2005estimation} is that via an integration by parts trick, the divergence can be rewritten in a more amenable form. Essentially, these regularity conditions allow us to rewrite the conditional score function $J(W) := J(\P_{W_0}(\cdot | s,a) \Vert \P_{W}(\cdot| s,a))$ as:
\begin{equation}\label{eq:sm-simplified}
J(W) = \frac{1}{2} \int_{\calS} \P_{W_0}(s'|s,a) \cdot \sum_{i=1}^{d_s} \insquare{ (\pai \log \P_{W}(s'|s,a))^2 + 2\pai^2 \log \P_{W}(s'|s,a)} ds' + C,
\end{equation}
where $C$ does not depend on the parameter $W$. In \Cref{apdx:sm-background} we provide a more formal derivation of (\ref{eq:sm-simplified}) for exponential family densities as well as further discussion on \Cref{assume:regularity}.

Crucially, (\ref{eq:sm-simplified}) \emph{can be estimated with samples without requiring computation of the partition function}, since the partition function vanishes when taking partial derivatives with respect to $s'$. This gives rise to the following formulation of the \textbf{empirical score matching loss} for a dataset $\calD =\{(s_t, a_t, s_t')\}_{t\in[n]}$:
\begin{equation*}\label{eq:esml}
    \hat{J}_n(W) := \frac{1}{2} \sum_{t=1}^n \sum_{i=1}^{d_s} \inparen{ (\pai \log \P_W(s_t'|s_t,a_t))^2 + 2\pai^2 \log \P_W(s_t'|s_t,a_t)}.\tag{\text{SM-L}}
\end{equation*}
Furthermore, for any regularizer $\lambda > 0$, we can define the \textbf{empirical score matching estimator}:
\begin{equation*}\label{eq:esme}
    \hat{W}_{n,\lambda} := \argmin_{W} \hat{J}_n(W) + \tfrac{\lambda}{2} \norm{W}_F^2. \tag{\text{SM-E}}
\end{equation*}  

The following theorem gives a closed form expression for the empirical score matching estimator, when specialized to densities given by \Cref{assume:ef-model}.

\begin{theorem}\label{thm:sm-loss}
For a dataset $\calD =\{(s_t, a_t, s_t')\}_{t\in[n]}$, \normalfont{(\ref{eq:esml})} can be written as:
\begin{align*}
    \hat{J}_n(W)= \frac{1}{2} \ip{\vecc{W}, \hat{V}_n \vecc{W}} + \ip{\vecc{W}, \hat{b}_n} + C,
\end{align*}
where:
\begin{align*}
    \hat{V}_n &:= \sum_{t=1}^n \sum_{i=1}^{d_s} \vecc{\pai \psi(s'_t) \phi(s_t,a_t)^\top} \vecc{\pai \psi(s'_t) \phi(s_t,a_t)^\top}^\top \in \R^{d_\psi d_\phi \times d_\psi d_\phi}, \\
    \hat{b}_n &:= \vecc{\sum_{t=1}^n \sum_{i=1}^{d_s} \inparen{\pai \log q(s'_t)\pai \psi(s'_t) + \pai^2 \psi(s'_t)} \phi(s_t,a_t)^\top} \in \R^{d_\psi d_\phi},
\end{align*}
and $C$ does not depend on $W$. In addition, \normalfont{(\ref{eq:esme})} can be computed as:
\begin{equation}\label{eq:sm-est}
    \vecc{\hat{W}_{n,\lambda}} = -(\hat{V}_n + \lambda I)^{-1} \hat{b}_n. 
\end{equation}
\end{theorem}
\Cref{thm:sm-loss} is a typical result in score matching literature, and can be derived as a corollary of \citet[Thm.~3]{arbel2018kernel}. For completeness, we give a proof in \Cref{apdx:pf-sm-loss}.

For the rest of the paper, it is useful to derive matrix expressions for $\hat{V}_n$ and $\hat{b}_n$. We define the following functions:
\begin{align*}
    &\Phi(s,a) := [E_{11} \phi(s,a), E_{12}\phi(s,a), \dots E_{ij}\phi(s,a), \dots E_{d_{\psi}\cdot d_\phi} \phi(s,a)]^\top \in \R^{d_{\psi} d_\phi \times d_\psi},\\
    &C(s') := \sum_{i=1}^{d_s} \pai \psi(s') \pai \psi(s')^\top \in \R^{d_\psi \times d_\psi},  \quad \xi(s') := \sum_{i=1}^{d_s} \pai \log q(s')\pai \psi(s') + \pai^2 \psi(s') \in \R^{d_\psi}.
\end{align*}
In addition, we use the subscript $t$ to denote the value of the above expressions on sample $(s_t, a_t, s_t')$. We succintly represent $\hat{V}_n = \sum_{t=1}^n \Phi_t C_t \Phi_t^\top$ and $\hat{b}_n = \sum_{t=1}^n \Phi_t \xi_t$.
\paragraph{Computational efficiency.}
We make a few remarks on the computation of the score matching estimator. From \Cref{thm:sm-loss}, we see that computing $\hat{W}_n$ does not require estimation of the log-partition function $Z_{sa}$. The objective is a \emph{quadratic} function in $W$, which we can solve for via \Cref{eq:sm-est}. However, \Cref{eq:sm-est} requires us to invert a $d_\phi d_\psi \times d_\phi d_\psi$ matrix, which takes time $O(d_\phi^3 d_\psi^3)$ and memory $O(d_\phi^2 d_\psi^2)$. This can be disappointing from a practical perspective, where the dimensionality of $\phi$ and $\psi$ can be large. Several additional considerations may remedy this. First, using the representer theorem, it is possible to show that $\hat{W}$ is the solution of a linear system of $n\cdot d_S$ variables, thus taking time $O(n^3 d_S^3)$ and space $O(n^2 d_S^2)$ \citep[][Thm.~1]{arbel2018kernel}. One can further reduce the dependence on $n$ using Nystr\"{o}m approximations \cite{sutherland2018efficient}. Second, if we are in the structured setting where $W_0 = \sum_{i=1}^d \theta_i A_i$, where $\theta\in \R^d$ is unknown but the matrices $A_i \in \R^{d_\psi\times d_\phi}$ are known. \Cref{thm:sm-loss} can be adapted to this setting, and solving for $\hat{\theta}_n$ will take time $O(d^3)$ and space $O(d^2)$.

\subsection{Concentration guarantee} 
We provide concentration guarantees for score matching under some structural assumptions:
\begin{assumption}[Structural scaling]\label{assume:bounds}~\
    \version{}{\vspace{-\topsep}}
    \begin{enumerate}[label={\bf (\Alph*)}]
        \item For any $(s,a) \in \calS \times \calA$ and $s'\sim\P_{W_0}(\cdot|s,a)$: we have $\xi(s')$ is $B_\psi$-subgaussian.
        \item For any $(s,a) \in \calS \times \calA$ and $s'\sim\P_{W_0}(\cdot|s,a)$: we have $C(s') W_0\phi(s,a)$ is $B_c$-subgaussian. 
        \item For any $s'\in \calS$: $\alpha_1 I \preceq C(s') \preceq \alpha_2 I$, where $\alpha_2 \ge \alpha_1 > 0$.
        \item For any $(s,a)\in \calS\times \calA$: $\E^{W_0}_{sa} \psi(s')\psi(s')^\top - \E^{W_0}_{sa} \psi(s') \E^{W_0}_{sa} \psi(s')^\top \le \kappa I$.
    \end{enumerate}
\end{assumption}

The conditions in \Cref{assume:bounds} are mostly adapted from prior work \cite{sriperumbudur2017density, arbel2018kernel, chowdhury21}, with suitable modifications to accomodate our non-i.i.d.\ setting. Notably, \Cref{assume:bounds} holds for nonLDS (when $\Sigma = \sigma^2 I$) with \version{:
\begin{equation}
    B_\psi = \sigma^{-6}, \quad B_c = 0, \quad \alpha_1 =\alpha_2 = \sigma^{-4}, \quad \kappa = \sigma^{-2}.
\end{equation}}{$B_\psi = \sigma^{-6}, B_c = 0, \alpha_1 =\alpha_2 = \sigma^{-4}$ and $\kappa = \sigma^{-2}$.}
Due to space considerations, we defer further discussion on \Cref{assume:bounds} to \Cref{apdx:a2-discussion}. 


We can prove the following concentration guarantee.

\begin{theorem}\label{thm:concentration}
    Suppose Assumptions \ref{assume:regularity} and \ref{assume:bounds} hold. Let $\{\calF_t\}_{t=1}^\infty$ be a filtration such that $(s_t, a_t)$ is $\calF_{t}$ measurable, $s'_t$ is $\calF_{t+1}$ measurable, and $s'_t \sim \P_{W_0}(\cdot|s_t, a_t)$.
    
    For any $\delta \in (0,1)$ and $\lambda > 0$, let:
   \[\beta_n := \sqrt{\tfrac{2(B_\psi + B_c) }{\alpha_1^2}} \cdot \sqrt{\log \tfrac{\det(\lambda^{-1} \hat{V}_n + I)^{1/2}}{\delta}} + \sqrt{\lambda} \norm{W_0}_F. \]
   With probability at least $1-\delta$, the score matching estimators of \normalfont{(\ref{eq:esme})} satisfy:
   \[\norm{\vecc{\hat{W}_{n,\lambda}} - \vecc{W_0}}_{\hat{V}_n + \lambda I} \le \beta_n, \ \text{ for all } n\in \N.\] 
\end{theorem}
\Cref{thm:concentration} is a \emph{self-normalized} concentration guarantee, since the parameter error is rescaled by a data-dependent term $\hat{V}_n + \lambda I$. The proof is provided in \Cref{apdx:concentration-proof}. The proof relies on the method of mixtures argument developed in the linear bandit literature \cite[see, \eg][]{abbasi2011improved,lattimore2020bandit}.

%% file: main-results.tex
\section{Algorithm and main result}\label{sec:alg}
In this section, we present our main results, which introduce the Score Matching for RL (SMRL) algorithm (\Cref{algorithm:smrl}) and provide regret guarantees. 

\subsection{Algorithm specification}
Our algorithm works as follows. In each episode $k=1,2,\dots,K$, we compute a elliptic confidence set $\calW_k$ centered at our score matching estimator. In particular, we consider the $n := (k-1)H$ state transitions $\calD = \{s_t, a_t, s_t'\}_{t=1}^n$ the agent has observed up until the beginning of episode $k$ and run the score matching estimator to get the prediction $\hat{W}_k := \argmin_{W} \hat{J}(W) + \frac{\lambda}{2} \lVert W \rVert_F^2$, via (\Cref{eq:sm-est}. 
In discussing our RL algorithm and its regret guarantees, we choose to index $\hat{W}$ and $\hat{V}$ by $k$ rather than $n$ to emphasize that these quantities are computed once per episode. We also drop the subscript $\lambda$ because it is fixed across the run of the algorithm.

Let $B_\star$ is some known upper bound on $\norm{W_0}_F$. We define the confidence set
\begin{equation}\label{eq:conf-set}
    \calW_k := \inbraces{W\in \R^{d_\psi\times d_\phi}: \norm{\vecc{\hat{W}_k} - \vecc{W}}_{\hat{V}_k + \lambda I} \le \beta_k},
\end{equation}
where
\begin{align*}
    \beta_k := \sqrt{\tfrac{2(B_\psi + B_c) }{\alpha_1^2}} \cdot \sqrt{\log \tfrac{2\det(\lambda^{-1} \hat{V}_k  + I)^{1/2}}{\delta}} + \sqrt{\lambda} B_\star.
\end{align*}

Once the agent computes the confidence set $\calW_k$, they observe a new state $s_1^k$ and compute an optimistic policy $\pi^k$ (line 5-6), which is the optimal policy with respect to the ``best model'' in $\calW_k$. As long as $W_0 \in \calW_k$, the optimistic planning procedure gives us an overestimate of the true value function $V_{\P,1}^{\star}(s_1^k)$, ensuring sufficient exploration of the MDP. Lastly, the agent runs policy $\pi^k$ on the MDP to collect a new trajectory of data, which is added to the dataset $\calD$.

\subsection{Computational complexity}\label{sec:cc} \Cref{algorithm:smrl} has two main components: model estimation (line 9) via score matching and optimistic planning (line 6). We have already discussed in \Cref{sec:sm} that the model estimation can be computed efficiently. Planning is a different story. Even planning with a \emph{known model}, \ie solving the problem $\pi^k=\arg\max_{\pi} V_{\P_W,1}^{\pi}(s_1^k)$, is already challenging without imposing further structure. However, it can be approximated with model predictive control \cite{mayne2014model, wagener2019online}. Furthermore, even with access to a planning oracle, \emph{optimistic planning} is known to be NP-hard in the worst case \cite{danistochastic}. In this work, we assume computational oracle access to the optimistic planner that solves (line 6) and leave developing efficient approximation algorithms to future work. One alternative to optimistic planning is to employ posterior sampling methods in conjunction with (approximate) planning oracles; the Bayesian regret can be theoretically analyzed using well-established techniques \cite[\eg][]{osband2014model,chowdhury21}. 

\begin{algorithm}[t]
    \caption{Score Matching for RL (SMRL)}
\begin{algorithmic}[1]\label{algorithm:smrl}
    \State \textbf{Input: } Regularizer $\lambda$ and constants $B_\psi, B_c, B_\star, \kappa,\alpha_1.$
    \State \textbf{Initialize: } starting confidence set $\calW_1 = \R^{d_\psi\times d_\phi}$, confidence widths $\{\beta_k\}_{k\ge 1}$, dataset $\calD = \emptyset$.
    \For{episode $k=1,2,3,\cdots, K$}
        \State \textbf{Planning: } 
        \State Observe initial state $s_1^k$
        \State Choose the optimistic policy: $\pi^k=\arg\max_{\pi} \max_{W\in \calW_k}V_{\P_W,1}^{\pi}(s_1^k)$
        \State \textbf{Execution: }
        \State Execute $\pi^k$ to get a trajectory $\{s_h^k, a_h^k, r_h^k, s_{h+1}^k\}_{h\in[H]}$, and add it to $\calD$.
        \State \textbf{\boldmath Solve for score matching estimator $\hat{W}_k = \argmin_{W} \hat{J}(W) + \frac{\lambda}{2} \norm{W}_F^2$ via (\ref{eq:sm-est})}
        \State \textbf{\boldmath Update confidence set $\calW_{k+1}$ via (\ref{eq:conf-set})}
        \EndFor
\end{algorithmic}
\end{algorithm}

\subsection{Regret guarantee}
We now provide our main result, which is a $\sqrt{T}$-regret guarantee on the performance of SMRL.

\begin{theorem}[SMRL Regret Guarantee]\label{thm:smrl-regret}
Suppose Assumptions \ref{assume:ef-model} and \ref{assume:bounds} hold. Set $\lambda := 1/B_\star^2$ and fix $\delta \in (0,1)$. Then with probability at least $1-\delta$:
\[\calR(K) \le C \sqrt{\gamma_{K+1} \cdot \inparen{\tfrac{\kappa (B_\psi + B_c)}{\alpha_1^3}\inparen{\gamma_{K+1} + \log \nfrac{1}{\delta}}  + \tfrac{\kappa}{\alpha_1} + H}} \cdot \sqrt{H^2 T},\]
where $C > 0$ is an absolute constant and $\gamma_{K+1} := \log\det (\lambda^{-1} \hat{V}_{K+1} + I)$.
If $\norm{\phi(s,a)} \le B_\phi$ for all $(s,a)$, then $\calR(K) \le \tilde{O}(d_\psi d_\phi \cdot \sqrt{H^3 T})$, where the $\tilde{O}$ hides log factors and $\poly(\kappa, B_\psi, B_c,\alpha_1^{-1})$.\info{GL: provided dependence on constants.}
\end{theorem}

The proof is presented in \Cref{apdx:proof-regret}. A few remarks are in order. Our regret guarantee depends on the number of model parameters $d_\psi\cdot d_\phi$ and not on the state and action space sizes, thus making our algorithm sample-efficient in large-scale environments where $\abs{\calS}$ and $\abs{\calA}$ are infinite. Additionally, it is easy to redo the analysis when the parameter matrix is structured, \ie $W_0 = \sum_{i=1}^d \theta_i A_i$, to see that the regret guarantee depends on $d$ instead of $d_\psi\times d_\phi$. Thus, we can recover the same regret guarantee of $\tilde{O}(d\sqrt{H^3 T})$ that \citeauthor{chowdhury21} provide. 

On the more technical side, in \Cref{thm:smrl-regret}, we require $\phi$ to be a bounded feature mapping, which linear dynamical systems do not satisfy in general (recall $\phi = [s,a]^\top$, and $s,a$ can have unbounded norm). We need this to provide a bound on a certain ``information gain'' quantity $\gamma_{k} = \log\det(\lambda^{-1} \hat{V}_{k} + I)$ \cite[\cf][]{srinivas2009gaussian, kakade2020information}; however, the bounded $\phi$ assumption can be substantially weakened because our proof only requires $\sum_{h=1}^H \norm{\phi_h}^2$ to be bounded in every episode with high probability. In particular, if one restricts to controllable policies which do not blow up norm of the state \citep[\eg][]{cohen2019learning}, then the information gain term can be bounded.

\version{We do not know of any sharp lower bounds in our setting. In fact, we are unaware of any lower bounds even for bounded reward LDS for finite horizon episodic MDPs. For online LQRs, \citeauthor{simchowitz2020naive} prove a lower bound of $\Omega\inparen{\sqrt{d_\psi^2 d_\phi T}}$ translated to our notation, but their result is shown for the single trajectory, average reward setting so it is not directly comparable. Additionally, \citeauthor{kakade2020information} suggest a lower bound of $\Omega(\sqrt{HT})$ from tabular MDPs, although this is not formally proven. We can also adapt the SMRL algorithm to reward estimation in bandits, with $d_\psi =1$, $d_\phi = d$ and $H = 1$, and we match the lower bound of $d\sqrt{T}$ in this setting \cite{danistochastic}.}{}

%% file: smvsmle.tex
\section{Score matching vs maximum likelihood estimation}\label{sec:compare}

In this section, we provide a detailed comparison of score matching with maximum likelihood approaches. First we compare for exponential family transitions of \Cref{assume:ef-model}; then we specialize our comparison for the nonLDS setting. Lastly, we provide numerical evidence to demonstrate a setting where (a variant of) SMRL is superior.

\subsection{General comparison for exponential family transitions}
Score matching and MLE can be viewed as complementary techniques for density estimation; we highlight the relative pros and cons of SMRL vs\ Exp-UCRL.

In general, Exp-UCRL can be applied to more settings than score matching, due to the fact that score matching requires regularity conditions (\Cref{assume:regularity}) that are needed for the derivation of (\ref{eq:sm-simplified}). In particular, we require $\calS$ to be a Euclidean space and the feature vector $\psi: \calS \to \R^{d_\psi}$ to be a twice-differentiable mapping. In this sense, the scope of SMRL is more limited than that of Exp-UCRL. For example, while tabular and factored MDPs can be modeled as exponential family transitions via the softmax parameterization,\footnote{There is a mild technical issue, since \Cref{assume:ef-model} cannot capture transitions with probability 0, so we must assume that the support of the transitions is known in advance. See the paper \cite{chowdhury21} for more details.} we cannot prove regret guarantees for SMRL due to the differentiability requirement. Since the MLE estimator of \citeauthor{chowdhury21} can be computed in $\poly(S,A)$ time, in the tabular and factored MDP settings we would prefer to run Exp-UCRL. 

Among models given by \Cref{assume:ef-model} where \emph{both} score matching and MLE can be applied, score matching is preferred because the estimator can be computed in closed form as the solution to a ridge regression problem, and elliptic confidence sets can be constructed around it using \Cref{thm:concentration}. For the MLE, this is not possible in general. \citeauthor{chowdhury21} implicitly define the estimator as the solution to the likelihood equations, and their confidence set is constructed in a complicated fashion, in terms of sums of KL divergences taken over the dataset. \version{Translated to our notation, their confidence sets are:
\begin{align*}
    \calW_k := \Big\{ W\in \R^{d_\psi\times d_\phi}: \sum_{k'=1}^{k-1}&\sum_{h=1}^H \dkl\inparen{\P_{\hat{W}_k}(\cdot|s_h^{k'}, a_h^{k'}) \Vert \P_{W}(\cdot|s_h^{k'}, a_h^{k'}) } \\
    & + \frac{\lambda}{2}\norm{\vecc{\hat{W}_k} - \vecc{W}}_2 \le \beta_k \Big\}, 
\end{align*}
where $\hat{W}_k$ is the MLE estimate after episode $k-1$ and $\beta_k$ is some appropriately defined width.}{}Thus, while we are unable to claim overall computational tractability of \Cref{algorithm:smrl} due to the computational difficulty of optimistic planning, score matching enables us to estimate model parameters efficiently, an improvement from Exp-UCRL.

We now compare the regret guarantee of \Cref{thm:smrl-regret} with previous results; the detailed calculations are deferred to \Cref{sec:compare-expucrl}. We achieve the same order-wise guarantee as \citeauthor{chowdhury21}(Thm.~2) of $\tilde{O}(d_\phi d_\psi \cdot \sqrt{H^3 T})$. In terms of problem constants, both bounds depend on $\sqrt{\kappa}$, but we (1) require the constants $B_\psi$ and $B_c$, (2) replace dependence on strict convexity of the log partition function with the parameter $\alpha_1$.

\subsection{Comparison with prior work for nonLDS}
Now we compare our results for SMRL with the results for Exp-UCRL (\citeauthor{chowdhury21}) and LC$^3$ (\citeauthor{kakade2020information}) for the nonLDS problem with bounded and known rewards. For simplicity we will take the transition noise to be $\gauss{0}{\sigma^2 I_{d_s}}$. We will also assume that $\norm{W_0}_F \le B_\star$ and that the feature vectors are bounded as $\norm{\phi(s,a)} \le B_\phi$ for all $(s,a)\in \calS\times \calA$. All three are similar UCRL-style algorithms, and we compare the parameter estimation, confidence sets, and regret guarantee.

\paragraph{Estimation and confidence set construction.} For nonLDS, score matching and MLE are equivalent estimators (see \Cref{prop:equiv} for a formal statement).
Thus, in all three algorithms, the parameter estimation \emph{procedure} is identical, up to rescaling of regularization parameter $\lambda$. To further facilitate comparison, we will hereafter fix the $\lambda$ of each algorithm such that the parameter estimation is the same as LC$^3$ (for any fixed dataset). Our choices are detailed in \Cref{sec:compare-nonlds}.

Once we have fixed the parameter $\lambda$ for each algorithm, the main distinction lies in the confidence set construction. While all three algorithms essentially utilize the same optimistic planning procedure, optimistic planning depends on the confidence sets constructed in each episode. The chosen policies and the resulting trajectories will be different in all three algorithms. The confidence sets constructed for each paper are essentially the tightest self-normalized bound one can prove, so it is hard to directly compare the confidence sets from paper to paper due to the difference in analyses. Generally speaking, SMRL uses Frobenius norm bounds (\Cref{thm:concentration}), Exp-UCRL uses a mixture of both Frobenius and spectral \cite[Sec.~3.1]{chowdhury21}, and LC$^3$ uses only spectral norm bounds \cite[Eq.~3.2]{kakade2020information}.

\paragraph{Regret guarantee.} In terms of the regret guarantee, \Cref{thm:smrl-regret} gives us a regret guarantee of $\tilde{O}\big(\sqrt{d_\phi d_\psi \cdot (\sigma^4 d_\phi d_\psi  + H) H^2 T}\big)$, while a bound of $\tilde{O}\big(\sqrt{d_\phi^2 d_\psi^2 (1+ \sigma^{-2} B_\star^2 B_\phi^2 H) H^2 T} \big)$ can be derived for Exp-UCRL. Note that the latter bound depends polynomially on the scale of $W_0$ and $\phi$. \citeauthor{kakade2020information} (Remark 3.5) give a bound for LC$^3$ of $\tilde{O}(\sqrt{d_\phi (d_\phi + d_\psi + H) H^2 T})$, without polynomial dependence on $\sigma^2$ and the scale of $W_0$ and $\phi$. We conjecture that the $\sigma^2$ dependence is an artifact of our analysis, but it is less clear whether the dependence on $d_\phi, d_\psi$ can be improved.


\subsection{Experiments on synthetic MDP}\label{sec:experiment}
\begin{figure}[t]
    \centering
    \subfigure[MDP transition and reward]{\includegraphics[width=0.32\linewidth]{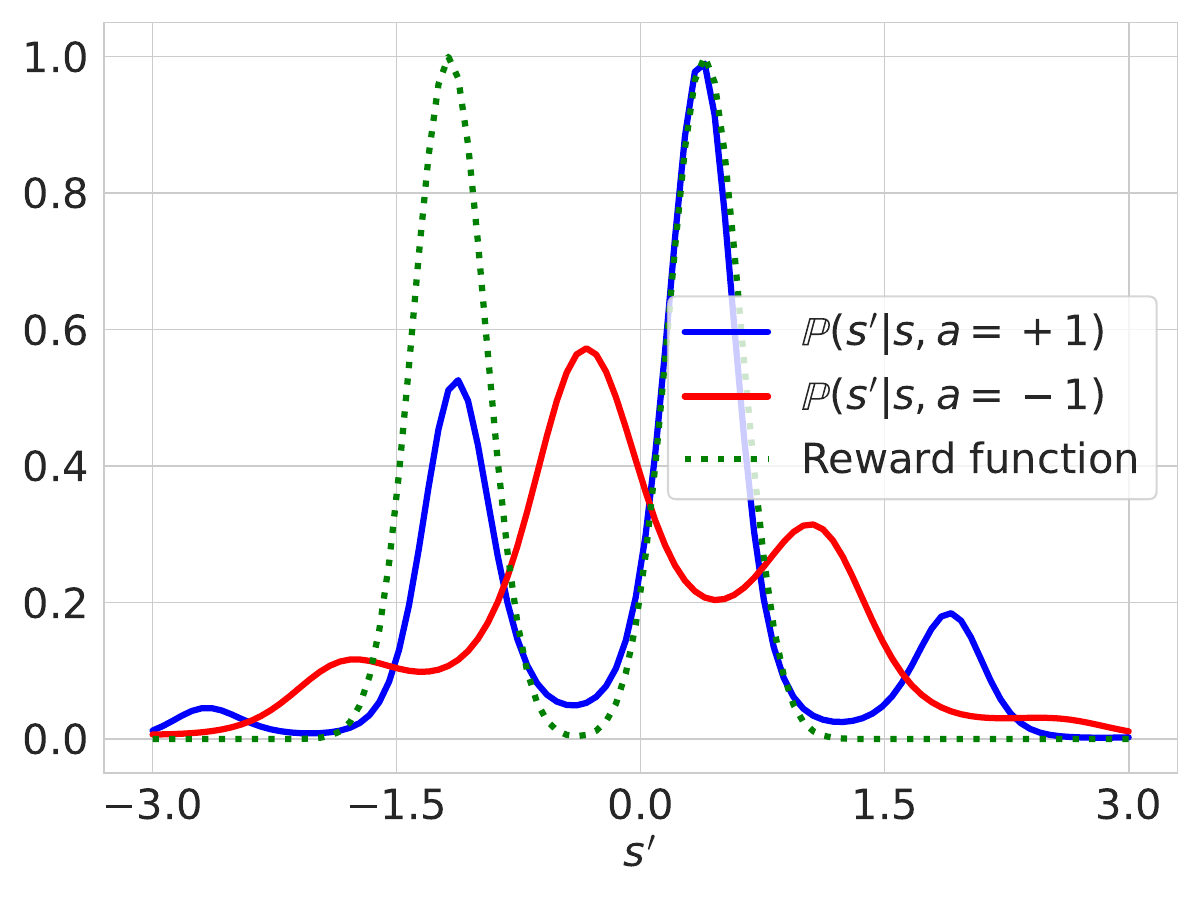}}
    \hfill
    \subfigure[Cumulative rewards]{\includegraphics[width=0.32\linewidth]{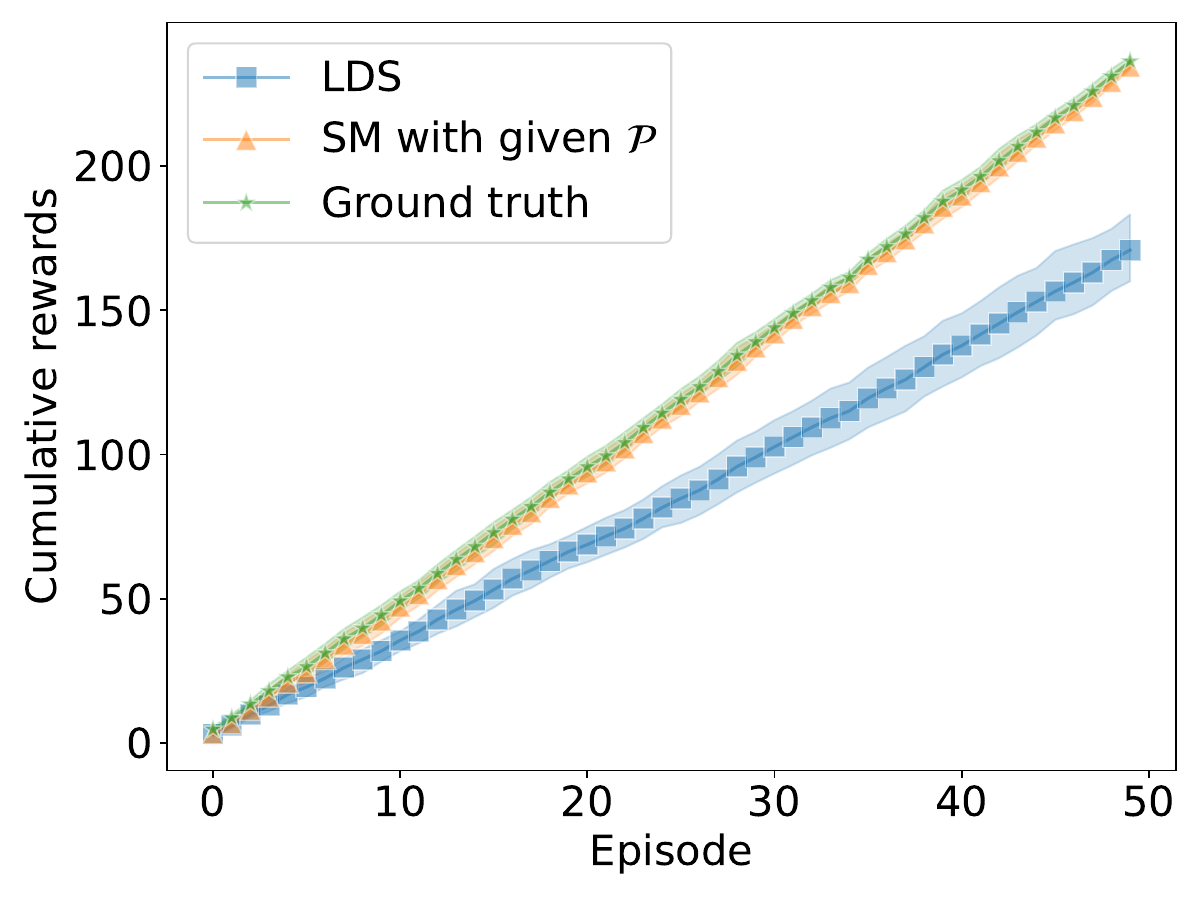}}
    \hfill
    \subfigure[Planner action choices]{\includegraphics[width=0.32\linewidth]{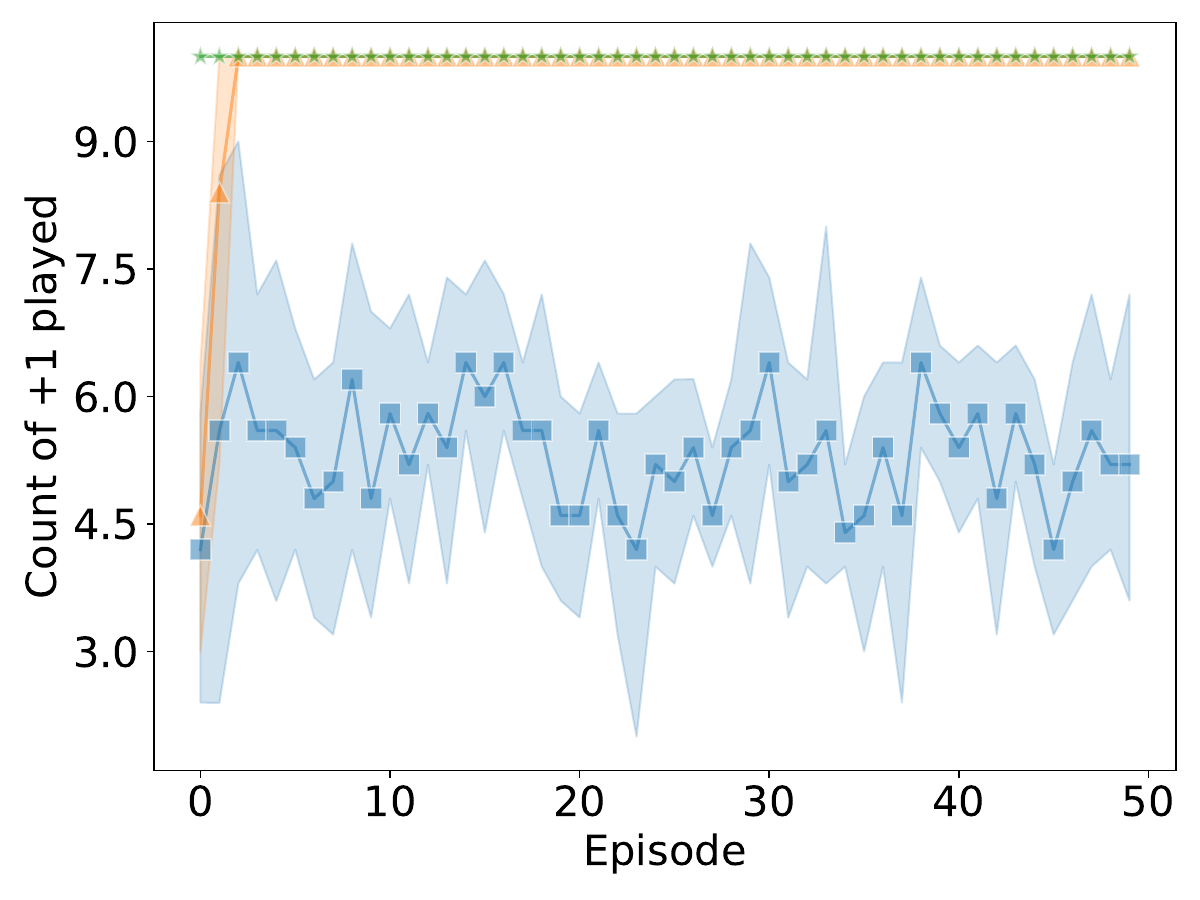}}
    \caption{Comparing SM vs fitting an LDS for a synthetic MDP, with $\calS = \R$, $\calA = \{+1, -1\}$,  $H=10$, initial state distribution $\mathrm{Unif}([-1,+1])$, $\P(s'|s,a) = \exp(-s'^{1.7}/1.7) \cdot \exp(\sin(4s') (s+a))$, and $r(s,a) = \exp(-10(s - \pi/8)^2) + \exp(-10(s + 3\pi/8)^2)$. (a) plots $\P$ for a single starting state $s=0.5$ for $a=+1$ and $a=-1$; the reward $r$ is superimposed. Taking $a=+1$ is more likely to transition to states with high reward. (b) plots cumulative reward for fixed planner with varying model estimation: SM with the given $\calP$, fitting an LDS, and a baseline with the ground truth model. (c) plots the number of steps in every episode where $a=+1$ is picked by the planner. In (b) and (c), shaded areas correspond to 95\% confidence intervals.}
    \label{fig:exp-syn-mdp}
    \vspace*{-1.5em}
\end{figure}

We demonstrate end-to-end benefits of using score matching in a (highly stylized) synthetic MDP; see \Cref{fig:exp-syn-mdp}. In our constructed MDP, the transition function is multimodal; the action choice affects the location of the modes of the next state density. The reward is constructed so that $a=+1$ leads to higher reward than $a=-1$ at most states. To enable fair comparison, we \emph{fix} a simple random sampling shooting planner \cite{rao2009survey} and evaluate three model estimation procedures: score matching with the given class $\calP$, fitting an LDS via MLE, and a baseline where planner is supplied the ground truth $\P$. (For this simple one-dimensional RL task, one can also numerically compute the MLE with the given $\calP$. However, this approach does not scale to RL tasks with high-dimensional states.) Fitting an LDS does poorly because the LDS density is not expressive enough to differentiate between $a=+1$ and $a=-1$, while score matching estimates the density well, so the planner quickly learns to pick $a=+1$. Our experiments suggest that modeling the transition $\P$ via the richer \Cref{assume:ef-model} can yield end-to-end benefits for RL tasks. Further experimental details can be found in \Cref{apdx:exp_detail}.

%% file: related-work.tex
\section{Related work}\label{sec:relatedwork}
In this section we discuss several related works on provably efficient reinforcement learning.

In the tabular setting, our theoretical understanding of RL is fairly complete \cite{jaksch2010near,osband2016generalization,azar2017minimax,dann2017unifying,agrawal2017optimistic}. The best possible regret is $\tilde{\Theta}(\sqrt{H^2 SAT})$ \cite{jaksch2010near,azar2017minimax}, where $S$ denotes the state space size and $A$ denotes the action space size. The minimax lower bound shows that unless stronger assumptions are placed on the MDP, we must incur regret that scales as $\Omega(\sqrt{SA})$, which can be exponentially large in real-world problems.

There has been considerable theoretical effort in understanding RL with specific model-based or model-free assumptions. Prior work often posits linearity assumptions in order to design algorithms which replace the dependence on $S,A$ with some notion of intrinsic dimensionality \cite{yang2019reinforcement, jin2020provably, cai2020provably, lattimore2020learning, zanette2020learning, modi2020sample}. In particular, \citet{yang2019reinforcement} consider a bilinear transition model; in comparison, our transition model given by \Cref{assume:ef-model} is \emph{log-bilinear}. In nonlinear settings, a line of work establishes regret guarantees or system identification guarantees for the LQR problem with unknown dynamics \cite{abbasi2011regret, dean2018regret,mania2019certainty, cohen2019learning, simchowitz2020naive} as well as its nonlinear generalization \cite{mania2020active,kakade2020information}. The transitions in the LQR are \emph{nonlinear}, but the corresponding value functions are linear in an appropriate basis. Our work studies a generalization of the transition dynamics in the aforementioned works; both the transition dynamics and the value functions can be nonlinear. 

Another line of work studies RL with general function approximation. Some authors study model-free RL by placing structural assumptions on the value function, \ie assuming bounded Bellman rank \cite{jiang2017contextual, dong2020root} or bounded eluder dimension \cite{wang2020reinforcement}. While in principle it is possible to apply model-free algorithms with general function approximation to our setting, the induced class of value functions generated by our assumption can be complex. More relevant to our work are model-based approaches \cite{sun2019model, osband2014model, ayoub2020model}. In contrast to our work, these papers do not focus on the computational tractability of model estimation. \citet{sun2019model} prove PAC-learning guarantees for model-based RL by relying on a complexity measure called the \emph{witness rank}. They analyze exponential family models under the assumption that the Hessian of the log-partition function has lower and upper-bounded eigenvalues. Our setting can be viewed as a special case of theirs with linear test functions $\calF = \{(s,a, s') \mapsto \ip{W \phi(s,a), \psi(s')}: W \in \R^{d_\psi \times d_\phi} \}$. Their results are not directly comparable to ours, since their final bound depends on the log cardinality of the model class $\calM$ and the test function class $\calF$, as well as a linear dependence on the number of actions; in contrast, our bound scales with the dimensionality $d_\psi\cdot d_\phi$, with no dependence on the number of actions. \citet{osband2014model} prove Bayesian regret guarantees for posterior sampling in terms of the \emph{eluder dimension} \cite{russo2013eluder} of the model class; their results are limited due to requirements of global Lipschitzness on the future value functions and subgaussianity of the transitions. \citet{ayoub2020model} introduce a model-based algorithm which learns a general class of transition models $\calP$ via a technique called \emph{value-targeted regression}. Their regret guarantees depend on the eluder dimension of a constructed $Q$ function class $\calF_\calP$. Lastly, \citet{foster2021statistical} introduce a \emph{Decision Estimation Coefficient} and show that it provides upper and lower bounds for interactive decision making.



%% file: conclusion.tex
\section{Conclusion}\label{sec:conclude}
In this paper, we show $\sqrt{T}$-regret guarantees for a reinforcement learning setting when the state transition model is an exponential family model, a challenging nonlinear setting. Under this modeling assumption, the commonly employed MLE may be intractable; we bypass such issues by proposing to learn the model via the score matching method.

We conclude with a few possible directions for future work. 

\version{}{\vspace{-\topsep}}
\version{\begin{itemize}}{\begin{itemize}[leftmargin=0.5cm]}
    \item \emph{Model Misspecification:} Proving guarantees for SMRL when the underlying transition $\P$ do not lie in the model class $\calP$ but instead is well-approximated by $\tilde{\P} \in \calP$ is an interesting direction.
    \item \emph{Arbitrary State Spaces:} A key limitation of the score matching estimator is that it requires that the state space $\calS$ must be a subset of the Euclidean space $\R^{d_s}$ and the feature mapping $\psi$ to be twice differentiable; therefore it cannot handle arbitrary state spaces. One important direction is extending the score matching algorithm to discrete state spaces such as tabular/factored MDPs through a suitable modification of the estimation procedure \cite[\eg][]{hyvarinen2007some,lyu2009interpretation}.
    \item \emph{Kernelization:} We would like to extend our guarantees to the \emph{kernel conditional exponential family} (KCEF) setting of \citet{arbel2018kernel}, \ie when the conditional model is $\P_f(s'|s,a) := q(s') \cdot \exp \inparen{\ip{f, \Gamma_{sa} k(s',\cdot)} - Z_{sa}(f)}$, where $f$ lies in some vector valued Reproducing Kernel Hilbert Space (RKHS) $\calH$, $k(s',\cdot)$ lies in an RKHS $\calH_S$, and $\Gamma_{sa}: \calH_S \to \calH$ is an operator that depends on $(s,a)$. This generalizes our finite dimensional setting and is a special case of the conditional family where the inner product is $\ip{f, \phi(s,a,s')}$, studied by \citet{canu2006kernel}. In the KCEF setting, MLE becomes computationally intractable; yet score matching can be kernelized and efficiently computed, and fast approximation methods exist \cite{sutherland2018efficient}. Our theory does not hold for the KCEF because the parameter $\alpha_1 = 0$. Instead, one might be able to adapt the range-space assumption from the paper \cite{arbel2018kernel} to the non-i.i.d.~setting.
\end{itemize}

%% file: apdx-sm.tex
\section{Score matching details}\label{apdx:sm}
In this section, we provide details on the score matching estimator and prove \Cref{thm:sm-loss} and \Cref{thm:concentration}. We use results from prior work \cite{hyvarinen2005estimation,sriperumbudur2017density,arbel2018kernel} and defer further discussion to those papers.

\subsection{Background on score matching}\label{apdx:sm-background}

We recall some notation.

The model we consider under \Cref{assume:ef-model} is:
\[
    \P_{W_0}(s'|s,a) = q(s') \exp\inparen{\ip{\psi(s'), W_0\phi(s,a)} - Z_{sa}(W_0)}, 
\]
where
\[
    W_0 \in \calW := \inbraces{W\in \R^{d_\psi \times d_\phi}: \int_\calS q(s') \exp(\ip{\psi(s'), W\phi(s,a)}) \ ds' < \infty, \ \forall (s,a)\in \calS\times \calA}.
\]
Let us define the population score matching loss for a fixed $(s,a)\in \calS\times \calA$ pair as:
\[J_{sa}(W) := J(\P_{W_0}(\cdot | s,a) \Vert \P_W(\cdot| s,a)) = \frac{1}{2} \int_{\calS} \P_{W_0}(s'|s,a) \norm{ \nabla_{s'} \log \frac{\P_{W_0}(s'|s,a) }{\P_{W}(s'|s,a)} }^2 ds'.\]


Hereafter, we will assume standard regularity conditions (\Cref{assume:regularity}) which enable us to employ integration by parts to simplify the score matching objective, \cite[\eg][Appendix A.4]{arbel2018kernel}. We pause to make a few remarks about \Cref{assume:regularity} and compare to previous work \cite{arbel2018kernel}. 
\version{}{\vspace{-\topsep}}
\version{\begin{itemize}}{\begin{itemize}[leftmargin=0.5cm]}
    \item It is fair to note that these regularity assumptions limit the applicability of score matching. Several exponential family densities whose parameters can be estimated via MLE cannot be estimated via score matching. For example, exponential distributions violate the boundary condition because the pdf does not decay to 0 as we approach 0. Other distributions like the gamma distribution violate the integrability conditions. For \emph{nonnegative distributions}, follow up work by \citet{hyvarinen2007some} provides a modified score matching estimator, which seems to work well empirically for density estimation of nonnegative gaussians; to the best of our knowledge, finite sample concentration guarantees for this estimator do not exist.
    \item Our distributional assumption is a special case of the kernel conditional exponential family introduced by \citet{arbel2018kernel}; so everywhere they have $k(s',\cdot)$ we replace it with $\psi(s')$.
    \item  We generally work with the conditional versions of the score matching loss $J_{sa}$, while \citet{arbel2018kernel} work in the i.i.d.\ setting, where they assume the existence of a joint distribution $p(x,y) = \pi(x) \cdot p(y|x)$. Thus, their theorems are stated for the averaged loss $J(W) := \int_\calS J_{x}(W) \pi(x)$. 
\end{itemize}

In addition, we will define the following quantities, which can be viewed as the population versions of various quantities which appear in \Cref{thm:sm-loss}:
\begin{align*}
\bar{V}_{sa} &:= \E^{W_0}_{sa}\insquare{\sum_{i=1}^{d_s} \vecc{\pai \psi(s') \phi(s,a)^\top} \vecc{\pai \psi(s') \phi(s,a)^\top}^\top} \ \in \R^{d_\psi d_\phi \times d_\psi d_\phi},\\
\bar{\xi}_{sa} &:= \E^{W_0}_{sa}\insquare{ \sum_{i=1}^{d_s}  \pai \log q(s_t')\pai \psi(s_t') + \pai^2 \psi(s_t')} \ \in \R^{d_\psi d_\phi}, \\
\bar{C}_{sa} &:= \E^{W_0}_{sa}\insquare{\pai \psi(s_t')\pai \psi(s_t')^\top} \ \in \R^{d_\psi\times d_\psi}.
\end{align*}

Now we introduce the main theorem for score matching, a version of which was shown in \citet[Theorem 3]{arbel2018kernel}. 

\begin{theorem} \label{thm:sm}
Under \Cref{assume:regularity}, the following are true for all $(s,a)\in \calS\times \calA$.
\version{}{\vspace{-\topsep}}
\version{\begin{enumerate}}{\begin{enumerate}[leftmargin=0.5cm]}
    \item $J_{sa}(W) < \infty$ for all $W\in \calW$.
    \item For all $W \in \calW$, $J_{sa}(W) = \frac{1}{2} \ip{\vecc{W- W_0}, \bar{V}_{sa} \vecc{W- W_0}}$.
    \item Alternatively:
    \[J_{sa}(W) = \frac{1}{2} \ip{\vecc{W}, \bar{V}_{sa} \vecc{W}} + \ip{\vecc{W}, \bar{\xi}_{sa}} + J(\P_{W_0}(\cdot |s,a)\Vert q(\cdot)),\]
    \item $\bar{\xi}_{sa}= -\bar{C}_{sa} W \phi(s,a)$.
\end{enumerate}
\end{theorem}

We do not formally prove this because all that has changed is some of the notation from \citet{arbel2018kernel}, as well as specializing their results to the finite-dimensional setting. In order to derive the alternative expression in \Cref{thm:sm}, part 3, we must use the integration by parts trick which was first introduced by \citet{hyvarinen2005estimation} and requires the boundary condition {\bf (C)} and integrability conditions {\bf (D)}. \Cref{thm:sm}, part 3 is useful as it gives us an expression for the score matching estimator that does not require us to estimate the log partition function; as we show in \Cref{thm:sm-loss}, the empirical loss is simply a quadratic function in $W$.

\subsection{Proof of \Cref{thm:sm-loss}}\label{apdx:pf-sm-loss}
For notational convenience, we will denote $\phi_t := \phi(s_t, a_t)$ and $\psi_t := \psi(s_t')$. We write the score matching loss as:
\begin{align*}
    \hat{J}_n(W) &:= \frac{1}{2} \sum_{t=1}^n \sum_{i=1}^{d_s} \inbraces{ (\pai \log \P_W(s_t'|s_t,a_t))^2 + 2\pai^2 \log \P_W(s_t'|s_t,a_t)}\\
    &= \frac{1}{2} \sum_{t=1}^n \sum_{i=1}^{d_s} \inbraces{\inparen{\pai \log q(s_t') + \pai \psi_t^\top W \phi_t}^2 + 2\inparen{\pai^2 \log q(s_t') + \pai^2 \psi_t^\top W \phi_t}}\\
    &= \frac{1}{2} \sum_{t=1}^n \sum_{i=1}^{d_s} \inparen{\pai \psi_t^\top W \phi_t}^2 + \sum_{t=1}^n \sum_{i=1}^{d_s} \inparen{\pai \log q(s_t') \pai \psi_t^\top W \phi_t + \pai^2 \psi_t^\top W \phi_t} + C,
\end{align*}    
where $C$ only contains terms that depend on $\pai q(s')$ and $\pai^2 q(s')$. The first part follows after using the trace trick identities: 
\begin{align*}
    a^\top M b &= \Tr(Mba^\top) = \ip{\vecc{M}, \vecc{ab^\top}} \\
    (a^\top M b)^2 &= \ip{\vecc{M}, \vecc{ab^\top} \vecc{ab^\top}^\top \vecc{M}},
\end{align*}
and applying linearity to move the sums inside the inner product operators.

The second part is simply the standard form of ridge regression estimator.

\subsection{Discussion of \Cref{assume:bounds}}\label{apdx:a2-discussion}
\begin{figure}[t]
    \centering
    \subfigure[$\alpha=1, \psi(s) = \exp(\sin(s))$]{\includegraphics[width=0.32\linewidth]{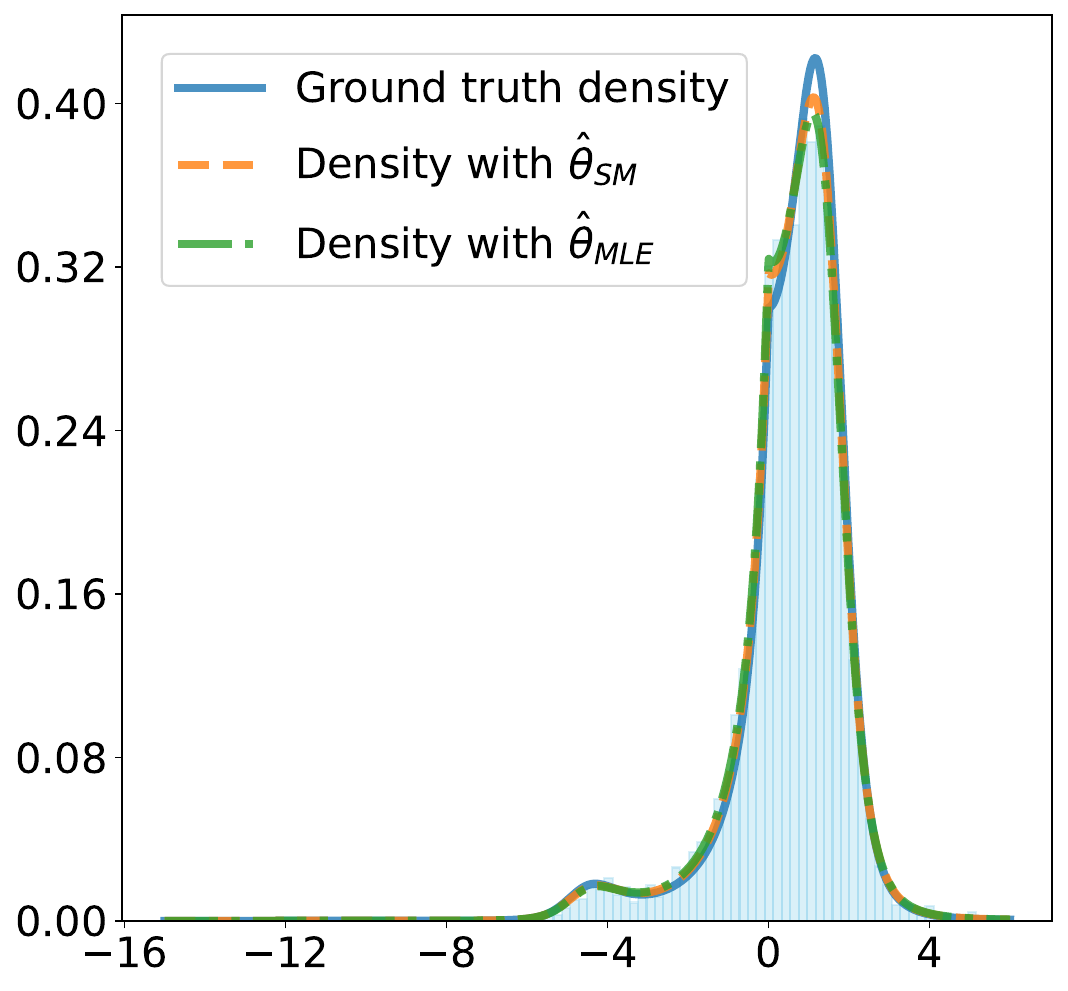}}
    \hfill
    \subfigure[$\alpha=2, \psi(s) = s \sin(s)$]{\includegraphics[width=0.32\linewidth]{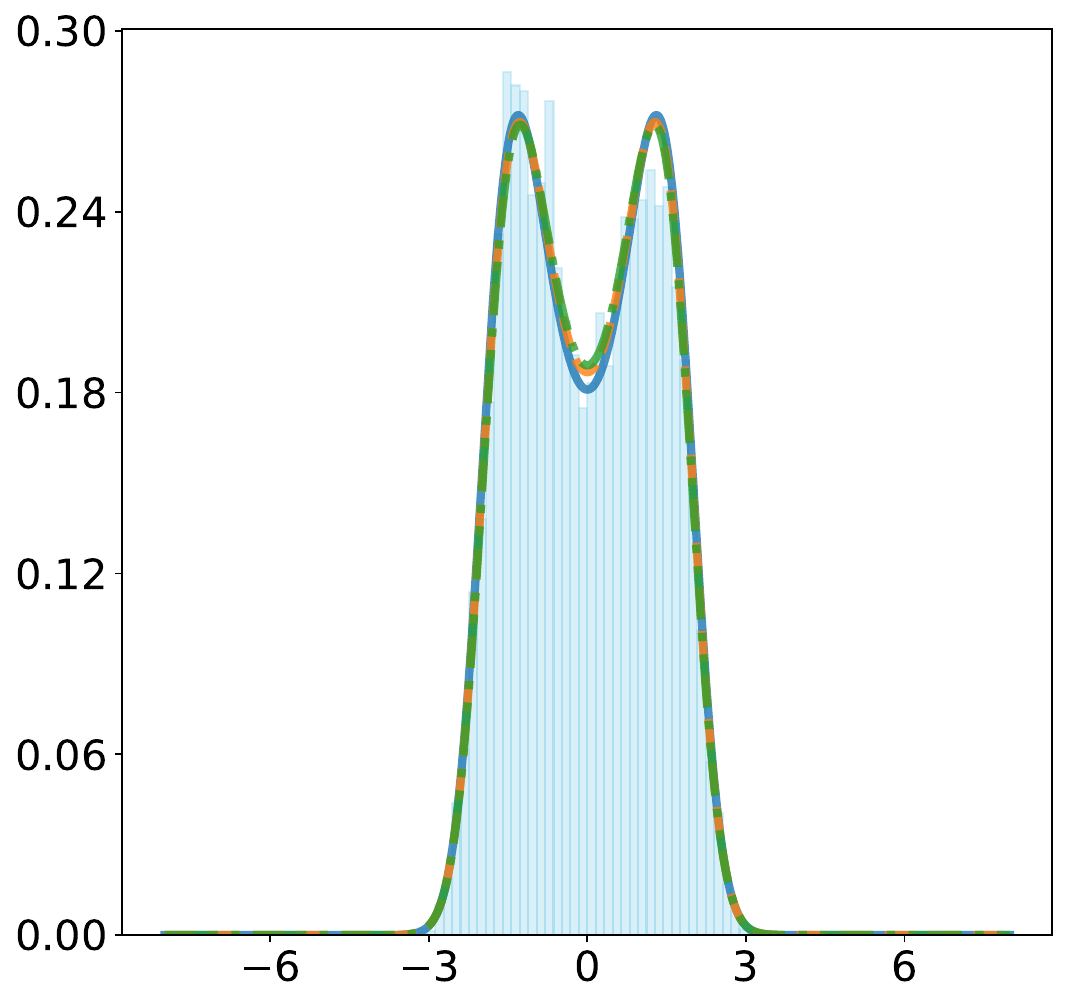}}
    \hfill
    \subfigure[$\alpha=1.7, \psi(s) = \sin(4s)$]{\includegraphics[width=0.32\linewidth]{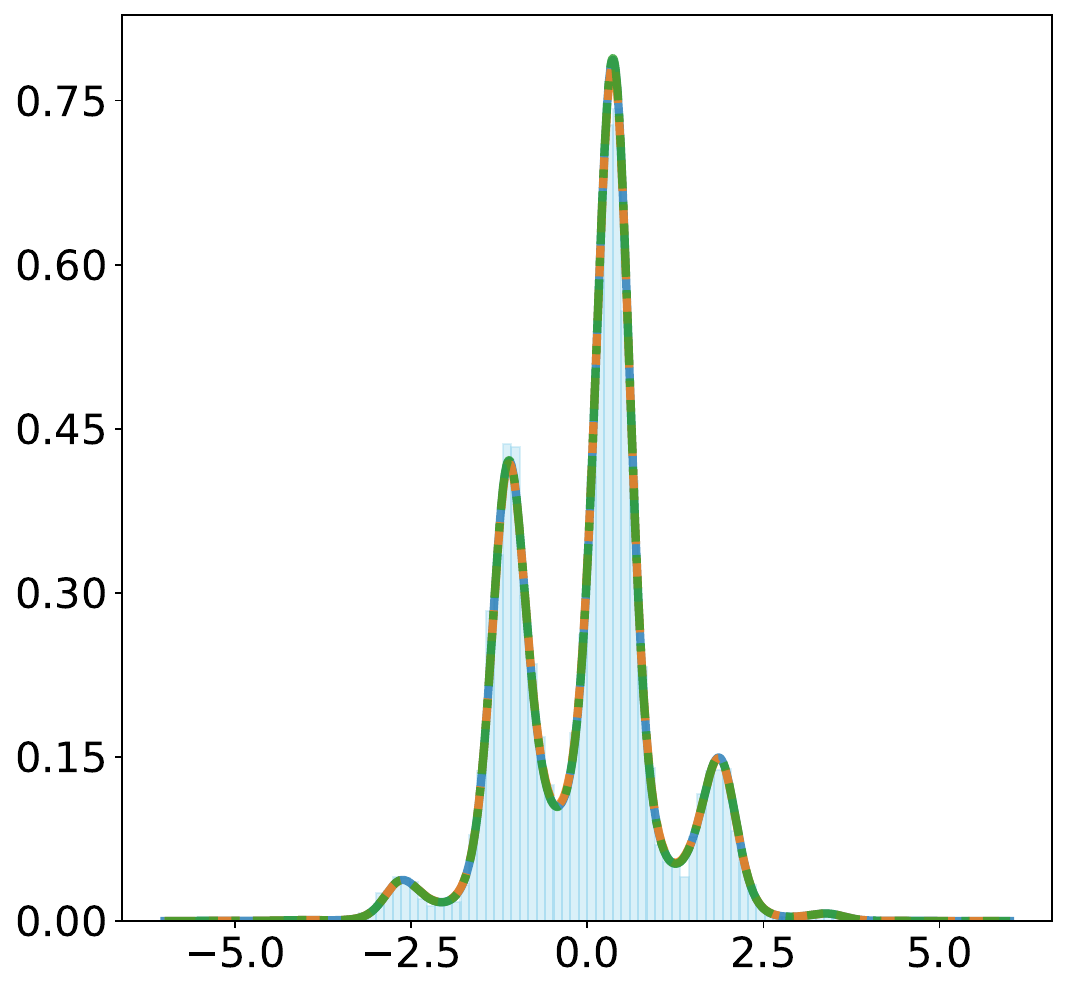}}
    \caption{Performance of score matching vs.~MLE for three unconditional 1D densities. We consider densities $q(s) \cdot \exp(\psi(s')\cdot \theta_0 - Z(\theta))$, where $q(s) = \exp(-s^\alpha/\alpha)$ for some $\alpha > 0$, $\psi(\cdot)$ is allowed to vary, and the parameter is set as $\theta_0=1$. For each density we sample $10^4$ points with Hamiltonian Monte Carlo, using the package \texttt{hamiltorch} \cite{cobb2019hamiltorch}. The MLE does not have a closed form and is approximated numerically. All densities violate \Cref{assume:bounds}(C), and possibly other assumptions.}
    \label{fig:1d_densities_plots}
    \vspace*{-1.5em}
\end{figure}

We give some additional background on \Cref{assume:bounds}. Roughly speaking, one can view the $\Phi_t$ as the covariates and $\xi_t$ as the response. The equation \Cref{eq:esme} bears strong resemblance to the ridge regression solution from the standard least squares setup; however, the main difference is that the ``covariance'' matrix $\hat{V}_n$ includes an additional term $C_t$ which captures the curvature of the $\psi$ mapping. {\bf (A)} and {\bf(B)}: are needed to control the regression error term $\hat{b}_n+\hat{V}_n \mathsf{vec}(W_0) = \sum_t \Phi_t(\xi_t + C_t W_0 \phi_t)$. {\bf (C)}: This assumption is a strengthening of a certain assumption in the i.i.d.~setting \cite{sriperumbudur2017density}, which assumes that $\ex{C(s')} \ge \alpha_1 I$. In the non-i.i.d.~ setting there is no fixed distribution that we draw $s'$ from, so we replace it with an almost sure bound $C(s') \succeq \alpha_1 I$. On a technical level, {\bf (C)} is required to change the matrix norm in the proof of \Cref{thm:concentration}. {\bf(D)}: this assumption essentially states that covariance of $\psi$ is bounded when drawn from the conditional distribution $\P(\cdot|s,a)$. It is used to analyze the regret guarantee to bound the KL divergence of the estimated density and the ground truth in terms of the parameter estimation error, and is also made in \citet{chowdhury21}.

Experimentally, we find that a much richer class of densities can be estimated
via score matching than required by \Cref{assume:bounds}, suggesting that \Cref{thm:concentration} can be shown under weaker conditions (see
\Cref{fig:1d_densities_plots}); we leave this to future work.





\subsection{Proof of \Cref{thm:concentration}}\label{apdx:concentration-proof}
Fix any $n\in \N$. Using the definition of $\hat{W}_{n,\lambda}$ from (\ref{eq:sm-est}), we can write:
\begin{align*}
    \norm{\vecc{\hat{W}_{n,\lambda}} - \vecc{W_0}}_{\hat{V}_n + \lambda I}
    &\le \norm{\hat{b}_n + \hat{V}_n\vecc{W_0}}_{(\hat{V}_n + \lambda I)^{-1}} + \lambda \norm{\vecc{W_0}}_{(\hat{V}_n + \lambda I)^{-1}}\\
    &\le \norm{\hat{b}_n + \hat{V}_n\vecc{W_0}}_{(\hat{V}_n + \lambda I)^{-1}} + \sqrt{\lambda} \norm{W_0}_F, \stepcounter{equation}\tag{\theequation}\label{eq:step1}
\end{align*}
where the first inequality uses triangle inequality and the second inequality uses the fact that $\hat{V}_n+\lambda I \succeq \lambda I$.

It now remains to bound the first term. We can write the first term as:
\begin{align*}
    \hat{b}_n + \hat{V}_n \vecc{W_0} = \sum_{t=1}^n  \Phi_t \xi_t + \Phi_t C_t \Phi_t^\top \vecc{W_0} = \sum_{t=1}^n  \Phi_t \underbrace{\inparen{\xi_t +C_t W_0 \phi_t}}_{=: \Delta_t}.
\end{align*}

From here, we have:
\begin{align*}
    \norm{\hat{b}_n + \hat{V}_n\vecc{W_0}}_{(\hat{V}_n + \lambda I)^{-1}} &= \norm{\sum_{t=1}^n  \Phi_t  \Delta_t }_{(\sum_{t=1}^n \Phi_t C_t \Phi_t^\top + \lambda I)^{-1}} \\
    &\le \alpha_{1}^{-1} \norm{\sum_{t=1}^n  \Phi_t  \Delta_t }_{(\sum_{t=1}^n \Phi_t \Phi_t^\top + \alpha^{-1}\lambda I)^{-1}}, \stepcounter{equation}\tag{\theequation}\label{eq:step2}
\end{align*}
using \Cref{assume:bounds}(iii). 

Next we show that $\Delta_t$ is conditionally subgaussian with parameter $B_\psi + B_c$. Then 
by \Cref{thm:sm} (iv), $\{\Delta_{t}\}_{t=1}^\infty$ is an $\calF_{t+1}$-adapted martingale difference sequence, since:
\[\ex{\Delta_t \given \calF_{t}} = \ex{ \sum_{i=1}^{d_s}  \pai \log q(s_t')\pai \psi(s_t') + \pai^2 \psi(s_t') + \pai \psi(s_t') \pai \psi(s_t')^\top W \phi(s_t,a_t) \given s_t,a_t} = 0.\]

Furthermore, using \Cref{assume:bounds}(i) and (ii) and applying \Cref{lem:sum-subg}, we see that $\Delta_t$ is a conditionally subgaussian random vector with parameter $B_\psi + B_c$. By \Cref{lem:self-normalized}, with probability at least $1-\delta$, we get that for all $n\in \N$:
\begin{align*}
    \norm{\sum_{t=1}^n  \Phi_t  \Delta_t }_{(\sum_{t=1}^n \Phi_t \Phi_t^\top + \alpha_1^{-1}\lambda I)^{-1}}^2 &\le 2(B_\psi + B_c) \log \frac{\det(\frac{\alpha_1}{\lambda} \sum_{t=1}^n \Phi_t \Phi_t^\top + I)^{1/2}}{\delta}\\
    &\le 2(B_\psi + B_c) \log \frac{\det(\lambda^{-1}\hat{V}_n + I)^{1/2}}{\delta}. \stepcounter{equation}\tag{\theequation}\label{eq:step3}
\end{align*}

The proof concludes by combining Eqs. \ref{eq:step1}, \ref{eq:step2}, and \ref{eq:step3}.

%% file: apdx-regret.tex
\section{Regret guarantee proof}\label{apdx:proof-regret}

In this section, we prove our main regret guarantee, \Cref{thm:smrl-regret}.

\paragraph{Notation.} For states and actions, we let $s_h^k$ and $a_h^k$ denote the state (action) that the agent observes (plays) in step $h$ of episode $k$. We let $\phi_h^k := \phi(s_h^k, a_h^k)$ and $\psi_h^k := \psi(s_{h+1}^{k})$. In addition, we denote $\Phi_{kh}$ and $C_{kh}$ for the matrix expressions of $\Phi$, $C$ for the data point $(s_h^k, a_h^k, s_{h+1}^k)$.

For value functions, we will generally write $V^{\pi}_{\P, h}$ to denote the value function of running policy $\pi$ in the MDP with transition model $\P$. Since we consider parameterized transition models, sometimes we will replace $\P$ with the parameter $W$. For the superscript, we adopt the following conventions: (1) we always denote $\pi^\star$ to be the optimal policy \emph{under the ground truth model $\P_{W_0}$}, and sometimes denote it as $\star$, (2) in episode $k$ with agent policy $\pi^k$, we replace $\pi^k$ with $k$. For example, we will read $V_{\hat{W},h}^{k}$ to be ``the value function of running policy $\pi^k$ on MDP parameterized by $\hat{W}$, at step $h$''. 

We will also define the natural filtration $\calF = \{\calF_h^k \}_{h\ge 1, k\ge 1}$, where:
\[\calF_h^k := \sigma\inparen{ \bigcup_{i \in [H],\ j \in [k-1]} \{s_i^j, a_i^j\} \ \cup \ \bigcup_{i \in [h],\ j = k} \{s_i^j, a_i^j\}},\]
representing all state-action pairs up to time $(k,h)$.

\subsection{Preliminary lemmas}We first introduce the auxiliary lemmas which we will use in our proof. First we introduce a recursive lemma that allows us to upper bound the regret.

\begin{lemma}[Recursive lemma]
Let $\tilde{W}_k =\argmax_{W \in \calW_k}V_{W,1}^{k}(s_1^k)$. Then:
\[ \sum_{k=1}^K \inparen{V_{\tilde{W}_k,1}^{k}(s_1^k)-V_{W_0,1}^{k}(s_1^k)}=\sum_{k=1}^K\sum_{h=1}^H \insquare{\E_{s_h^k a_h^k}^{\tilde{W}_k}V_{\tilde{W}_k,h+1}^{k}(s')-\E_{s_h^k a_h^k}^{W_0}V_{\tilde{W}_k,h+1}^{k}(s')} +\sum_{k=1}^K\sum_{h=1}^Hm_h^k,\]
where
$m_h^k=\E_{s_h^k a_h^k}^{W_0}\insquare{V_{\tilde{W}_k,h+1}^{k}(s')-V_{W_0,h+1}^{k}(s')}-\inparen{V_{\tilde{W}_k,h+1}^{k}(s_{h+1}^k)-V_{W_0,h+1}^{k}(s_{h+1}^k)}$ is a martingale difference sequence adapted to $\calF$ satisfying $|m_h^k|\leq 2H$.
\label{lemma: recursive}
\end{lemma}
\begin{proof}
For every $(k,h)\in [K]\times[H]$, we have 
\begin{align*}
&V_{\tilde{W}_k,h}^{k}(s_h^k)-V_{W_0,h}^{k}(s_h^k) \\
= \ &r(s_h^k,a_h^k)+\E_{s_h^k a_h^k}^{\tilde{W}_k}V_{\tilde{W}_k,h+1}^{k}(s')-r(s_h^k,a_h^k)-\E_{s_h^k a_h^k}^{W_0}V_{W_0,h+1}^{k}(s')\\
= \ &\E_{s_h^k a_h^k}^{\tilde{W}_k}V_{\tilde{W}_k,h+1}^{k}(s') -\E_{s_h^k a_h^k}^{W_0}V_{W_0,h+1}^{k}(s')\\
= \ &V_{\tilde{W}_k,h+1}^{k}(s_{h+1}^k)-V_{W_0,h+1}^{k}(s_{h+1}^k)+\inparen{\E_{s_h^k a_h^k}^{\tilde{W}_k}V_{\tilde{W}_k,h+1}^{k}(s')-\E_{s_h^k a_h^k}^{W_0}V_{\tilde{W}_k,h+1}^{k}(s')}\\
& \ \underbrace{-\inparen{V_{\tilde{W}_k,h+1}^{k} (s_{h+1}^k) - V_{W_0,h+1}^{k} (s_{h+1}^k)} + \E_{s_h^k a_h^k}^{W_0}V_{\tilde{W}_k,h+1}^{k}(s')-\E_{s_h^k a_h^k}^{W_0}V_{W_0,h+1}^{k}(s')}_{=: m_h^k},
\end{align*}
where in the first equality we used the definition of the value function and that $a_h^k = \pi^k(s_h^k)$. Note that because $s_{h+1}^k \sim \P_{W_0}(\cdot|s_h^k, a_h^k)$, $m_h^k$ is zero mean conditioned on $\calF_h^k$, so the martingale difference sequence property holds. Also, since $V \in [0,H]$, the a.s. bound holds.

Using this formula recursively we can obtain the lemma.
\end{proof}

The next lemma allows us to convert a bound on KL divergence to a self-normalized bound on the parameter error; a version of it is also shown in \citet{chowdhury21}.
\begin{lemma}[KL Divergence Bound]
    Under \Cref{assume:ef-model} and \Cref{assume:bounds}, for any $(s,a)\in\calS\times\calA$ and $W,W' \in\R^{d_\psi\times d_\phi}$, it holds that
    \[\dkl\inparen{\P_{W}(\cdot|s,a) \Vert\P_{W'}(\cdot|s,a)}\leq \frac{\kappa}{2} \norm{ \vecc{W} - \vecc{W'} }_{\Phi_{s,a}\Phi_{s,a}^\top}^2.\]
    \label{lemma: KL}
\end{lemma}
    \begin{proof}
    We have 
    \begin{align*}
    \dkl\inparen{\P_{W}(\cdot|s,a) \Vert \P_{W'}(\cdot|s,a)} &=\int _{\calS}\P_{W} (s'|s,a)\log\frac{\P_{W}(s'|s,a)}{\P_{W'}(s'|s,a)} \ ds'\\
    &=\E_{sa}^W \insquare{\psi(s')^\top (W - W')\phi(s,a) - Z_{sa}(W)+Z_{sa}(W')}.
    \end{align*}
    By Taylor expansion, there exists some $\tilde{W}$ which lies between $W$ and $W'$ that satisfies:
    \begin{align*}
        Z_{sa}(W') &= Z_{sa}(W) + \nabla_W Z_{sa}(W)^\top \vecc{W'- W} \\
        &\quad + \frac{1}{2} \vecc{W' - W}^\top \nabla_W^2 Z_{sa}(\tilde{W}) \vecc{W'-W},
    \end{align*}
    where $\nabla Z_{sa}(\cdot)$ and $\nabla^2 Z_{sa}(\cdot)$ are understood to be $\R^{d_\psi d_\phi}$ and $\R^{d_\psi d_\phi \times d_\psi d_\phi}$ respectively, representing the gradient and Hessian of the function $Z_{sa}: R^{d_\psi d_\phi} \to \R$.
    \Cref{lem:logpartitionderivatives} gives us expressions for $\nabla Z_{sa}(\cdot)$ and $\nabla^2 Z_{sa}(\cdot)$, which we can plug in to get:
    \begin{align*}
        \dkl\inparen{\P_{W}(\cdot|s,a) \Vert \P_{W'}(\cdot|s,a)} &\le \frac{1}{2} \vecc{W' - W}^\top \nabla_W^2 Z_{sa}(\tilde{W}) \vecc{W'-W} \\
        &\le \frac{\kappa}{2} \norm{ \vecc{W} - \vecc{W'} }_{\Phi_{s,a}\Phi_{s,a}^\top}^2,
    \end{align*}
    using \Cref{assume:bounds} (iv) in the last step. This proves the theorem.
\end{proof}

\begin{lemma}\label{lem:logdet} For any sequence of $\inbraces{(\Phi_{kh}, C_{kh})}_{k\in[K],h\in[H]}$, we have:
    \[\sum_{k=1}^K \min\inbraces{ \sum_{h=1}^H\norm{(\hat{V}_k + \lambda I)^{-1/2} \Phi_{kh} C_{kh} \Phi_{kh}^\top (\hat{V}_k + \lambda I)^{-1/2}}, 1}\le 2\log\det \inparen{\frac{\hat{V}_{K+1}}{\lambda} + I}.\]
    \end{lemma}
    \begin{proof}First we use the fact that for all $x\in [0,1]$, $x \le 2\log (1+x)$:
    \begin{align*}
        &  \min\inbraces{ \sum_{h=1}^H\norm{(\hat{V}_k + \lambda I)^{-1/2} \Phi_{kh} C_{kh} \Phi_{kh}^\top (\hat{V}_k + \lambda I)^{-1/2}}, 1} \\
        &\le 2 \log \inparen{1+ \sum_{h=1}^H\norm{(\hat{V}_k + \lambda I)^{-1/2} \Phi_{kh} C_{kh} \Phi_{kh}^\top (\hat{V}_k + \lambda I)^{-1/2}}} \\
        &= 2 \log \inparen{1+ \sum_{h=1}^H\norm{(\hat{V}_k + \lambda I)^{-1/2}\inparen{\Phi_{kh}\sum_{i=1}^{d_s}  \pai \psi_h^k \ \pai \psi_h^{k\top} \Phi_{kh}^\top}(\hat{V}_k + \lambda I)^{-1/2}}} \\
        &\le 2 \log \inparen{1+ \sum_{h=1}^H\sum_{i=1}^{d_s} \norm{(\hat{V}_k + \lambda I)^{-1/2}\inparen{\Phi_{kh} \pai \psi_h^k \ \pai \psi_h^{k\top} \Phi_{kh}^\top}(\hat{V}_k+\lambda I)^{-1/2}}} \\
        &= 2 \log \inparen{1+ \sum_{h=1}^H\sum_{i=1}^{d_s} \Tr\inparen{(\hat{V}_k + \lambda I)^{-1/2}\inparen{\Phi_{kh} \pai \psi_h^k \ \pai \psi_h^{k\top} \Phi_{kh}^\top}(\hat{V}_k+\lambda I)^{-1/2}}} \\
        &= 2 \log \inparen{1+ \Tr\inparen{ (\hat{V}_k+\lambda I)^{-1/2} \sum_{h=1}^H \inparen{\Phi_{kh}\sum_{i=1}^{d_s}  \pai \psi_h^k \ \pai \psi_h^{k\top} \Phi_{kh}^\top}(\hat{V}_k+\lambda I)^{-1/2}}} \\
        &\le 2  \log\det  \inparen{I+  (\hat{V}_k+\lambda I)^{-1/2} \sum_{h=1}^H \inparen{\Phi_{kh}\sum_{i=1}^{d_s}  \pai \psi_h^k \ \pai \psi_h^{k\top} \Phi_{kh}^\top}(\hat{V}_k+\lambda I)^{-1/2}} \\
        &= 2 \log\det \inparen{\hat{V}_{k+1} + \lambda I} - 2 \log\det \inparen{\hat{V}_{k} + \lambda I}. 
    \end{align*}
    
    In the first equality, we use the fact that the trace of a rank 1 matrix is equality to its spectral norm. In the last inequality, we used the fact that for any PSD matrix $A$, $\log\det (I + A) \ge \log (1+ \Tr(A))$. In the last line, we used the definition of $\hat{V}_{k}$ (recall it is computed using the first $n = (k-1)H$ samples):
    \begin{align*}
    & \log\det (\hat{V}_{k+1} + \lambda I) \\
    = \ &\log\det\inparen{\hat{V}_{k}  + \lambda I + \sum_{h=1}^H \Phi_{kh}C_{kh}\Phi_{kh}^\top} \\
    = \ &\log\det(\hat{V}_{k} + \lambda I) \\
    & + \log\det \inparen{I+  (\hat{V}_k+\lambda I)^{-1/2} \sum_{h=1}^H \inparen{\Phi_{kh}\sum_{i=1}^{d_s}  \pai \psi_h^k \ \pai \psi_h^{k\top} \Phi_{kh}^\top}(\hat{V}_k+\lambda I)^{-1/2}} .  
    \end{align*}
    Therefore, telescoping the sum we have:
    \begin{align*}
        &\sum_{k=1}^K \min\inbraces{ \sum_{h=1}^H\norm{(\hat{V}_k + \lambda I)^{-1/2} \Phi_{kh} C_{kh} \Phi_{kh}^\top (\hat{V}_k + \lambda I)^{-1/2}}, 1} \\
        &\le 2 \sum_{k=1}^K  \log\det \inparen{\hat{V}_{k+1} + \lambda I} - \log\det \inparen{\hat{V}_{k} + \lambda I} = 2\log\det\inparen{\frac{\hat{V}_{K+1}}{\lambda} + I}.
    \end{align*}
    \end{proof}

\subsection{Regret proof}
\begin{proof}[Proof of \Cref{thm:smrl-regret}]
By \Cref{thm:concentration}, the event $\calE:=\inbraces{W_0 \in \calW_k,\forall \ k \in [K]}$ holds with probability at least $1-\delta/2$. Henceforth, suppose $\calE$ holds.

In episode $k$, we pick $\pi^k=\argmax_{\pi} \max_{W\in \calW_k} V_{W,1}^\pi(s_1^k)$. Let us denote $\tilde{W}_k =\argmax_{W \in \calW_k}V_{W,1}^{k}(s_1^k)$, that is, $\tilde{W}$ is the ``optimistic model'' which $\pi^k$ is greedy with respect to. Under event $\calE$, we have $V^{\star}_{W_0, 1}(s_1^k) \le V^{k}_{\tilde{W}_k, 1}(s_1^k)$, so we have:
\[\calR(K)= \sum_{k=1}^K \insquare{V_{W_0,1}^{\star}(s_1^k) - V_{W_0,1}^{k}(s_1^k)}\leq \sum_{k=1}^K \insquare{V_{\tilde{W}_k,1}^{k}(s_1^k)-V_{W_0,1}^{k}(s_1^k)}. \]

Now we invoke \Cref{lemma: recursive} to get
\begin{align}
\mathcal{R}(K)\leq\sum_{k=1}^K\sum_{h=1}^H\insquare{\E_{s_h^k a_h^k}^{\tilde{W}_k}V_{\tilde{W}_k,h+1}^{k}(s')-\E_{s_h^k a_h^k}^{W_0}V_{\tilde{W}_k,h+1}^{k}(s')}+\sum_{k=1}^K\sum_{h=1}^Hm_h^k.
\label{eq:regret-bound}
\end{align}
Since $m_h^k$ is a martingale difference sequence with respect to $\calF$ that is a.s. bounded by $2H$, by Azuma-Hoeffding we have:
\[ \sum_{k=1}^K\sum_{h=1}^Hm_h^k\leq 2H\sqrt{2T \log\inparen{\frac{2}{\delta}}}, \]
with probability at least $1-\delta/2$. We now proceed to bound the first term of (\ref{eq:regret-bound}). For every $(k,h)\in [K]\times [H]$:
\begin{align*}
\E_{s_h^k a_h^k}^{\tilde{W}_k}V_{\tilde{W}_k,h+1}^{k}(s')-\E_{s_h^k a_h^k}^{W_0}V_{\tilde{W}_k,h+1}^{k}(s') &\le H\cdot\textrm{TV}\inparen{\P_{W_0}(\cdot|s_h^k,a_h^k,) \Vert \P_{\tilde{W}_k}(\cdot|s_h^k,a_h^k)} \\
&\le H \cdot \min \inbraces{\sqrt{\frac{1}{2} \dkl \inparen{\P_{W_0}(\cdot|s_h^k,a_h^k) \Vert \P_{\tilde{W}_k}(\cdot|s_h^k,a_h^k)}} , 1},\stepcounter{equation}\tag{\theequation}\label{eq:regret-step1}
\end{align*}
where the first inequality uses the TV bound (\Cref{lemma: TV}) and the second inequality is Pinsker's inequality (with the additional observation that TV distance is always bounded by 1).

Next we proceed to upper bound the KL divergence in terms of the parameter estimation error of score matching.

Using Lemma \ref{lemma: KL} and the triangle inequality, we have
\begin{align*}
&\sqrt{\dkl \inparen{\P_{W_0}(\cdot|s_h^k,a_h^k) \Vert \P_{\tilde{W}_k}(\cdot|s_h^k,a_h^k)}} 
\le \sqrt{\frac{\kappa}{2}} \norm{\vecc{W_0}-\vecc{\tilde{W}_k}}_{\Phi_{kh}\Phi_{kh}^\top} \\
& \le \sqrt{\frac{\kappa}{2}}\inparen{\norm{\vecc{W_0}-\vecc{\hat{W}_k}}_{\Phi_{kh}\Phi_{kh}^\top} + \norm{\vecc{\hat{W}_k}-\vecc{\tilde{W}_k}}_{\Phi_{kh}\Phi_{kh}^\top}},\stepcounter{equation}\tag{\theequation}\label{eq:regret-step2}
\end{align*}
where we recall that $\hat{W}_k$ is the output of the score matching estimator at round $k$, a.k.a.\ the center of our confidence ellipsoid. Under $\calE$, we know that $W_0, \tilde{W}_k \in \calW_k$.

Now we apply our concentration guarantee \Cref{thm:concentration}. For any $W \in \calW_k$, we know that:
\begin{align*}
&\norm{\vecc{W} - \vecc{\hat{W}_k}}_{\Phi_{kh}\Phi_{kh}^\top }^2 \\
= &\inparen{\vecc{W} - \vecc{\hat{W}_k} }^\top \Phi_{kh}\Phi_{kh}^\top \inparen{\vecc{W} - \vecc{\hat{W}_k} } \\
\le \ &\alpha_1^{-1} \inparen{\vecc{W} - \vecc{\hat{W}_k} }^\top \Phi_{kh} C_{kh} \Phi_{kh}^\top \inparen{\vecc{W} - \vecc{\hat{W}_k} }\\
\le \ &\alpha_1^{-1} \inparen{\vecc{W} - \vecc{\hat{W}_k} }^\top (\hat{V}_k + \lambda I)^{1/2}(\hat{V}_k + \lambda I)^{-1/2} \Phi_{kh} C_{kh} \Phi_{kh}^\top \\
\ &(\hat{V}_k + \lambda I)^{-1/2}(\hat{V}_k + \lambda I)^{1/2} \inparen{\vecc{W} - \vecc{\hat{W}_k} }\\
\le \ &\alpha^{-1}_1 \norm{\vecc{W} - \vecc{\hat{W}_k} }_{\hat{V}_k+\lambda I}^2 \cdot \norm{(\hat{V}_k + \lambda I)^{-1/2} \Phi_{kh} C_{kh} \Phi_{kh}^\top (\hat{V}_k + \lambda I)^{-1/2}}\\
\le \ &\alpha^{-1}_1 \beta_k^2 \cdot \norm{(\hat{V}_k + \lambda I)^{-1/2} \Phi_{kh} C_{kh} \Phi_{kh}^\top (\hat{V}_k + \lambda I)^{-1/2}}\\
\le \ &\alpha^{-1}_1 \beta_{K+1}^2 \cdot \norm{(\hat{V}_k + \lambda I)^{-1/2} \Phi_{kh} C_{kh} \Phi_{kh}^\top (\hat{V}_k + \lambda I)^{-1/2}}, \stepcounter{equation}\tag{\theequation}\label{eq:regret-step3}
\end{align*}
where the first inequality follows by \Cref{assume:bounds} (iii), the third inequality by Cauchy-Schwarz, the fourth inequality because $W\in \calW_k$, and the fifth inequality by the fact that $\beta_k$ is monotonically increasing in $k$.

Thus, by combinine (\ref{eq:regret-step1}), (\ref{eq:regret-step2}), and (\ref{eq:regret-step3}), we bound the first term of (\ref{eq:regret-bound}) as:
\begin{align*}
&\sum_{k=1}^K\sum_{h=1}^H\insquare{\E_{s_h^ka_h^k}^{\tilde{W}_k}V_{\tilde{W}_k,h+1}^{k}(s')-\E_{s_h^ka_h^k}^{W_0}V_{\tilde{W}_k,h+1}^{k}(s')} \\
\le \ &\sum_{k=1}^K \sum_{h=1}^H H\min\inbraces{\beta_{K+1} \cdot \sqrt{\frac{\kappa} {\alpha_1} \norm{(\hat{V}_k + \lambda I)^{-1/2} \Phi_{kh} C_{kh} \Phi_{kh}^\top (\hat{V}_k + \lambda I)^{-1/2}}},1 } \\
\le \ &\sum_{k=1}^K H^{3/2} \cdot \sqrt{\sum_{h=1}^H \min\inbraces{\beta_{K+1}^2 \cdot \frac{\kappa}{\alpha_1}\norm{(\hat{V}_k + \lambda I)^{-1/2} \Phi_{kh} C_{kh} \Phi_{kh}^\top (\hat{V}_k + \lambda I)^{-1/2}}, 1}}\\
\le \ &\sum_{k=1}^K H^{3/2} \cdot \sqrt{\min\inbraces{\beta_{K+1}^2 \cdot \frac{\kappa}{\alpha_1} \sum_{h=1}^H\norm{(\hat{V}_k + \lambda I)^{-1/2} \Phi_{kh} C_{kh} \Phi_{kh}^\top (\hat{V}_k + \lambda I)^{-1/2}}, H}}\\
\le \ &\sum_{k=1}^K H^{3/2} \cdot \sqrt{\max\inbraces{\beta_{K+1}^2 \cdot \frac{\kappa}{\alpha_1}, H} \cdot \min\inbraces{ \sum_{h=1}^H\norm{(\hat{V}_k + \lambda I)^{-1/2} \Phi_{kh} C_{kh} \Phi_{kh}^\top (\hat{V}_k + \lambda I)^{-1/2}}, 1}}\\
\le \ &H^{3/2} \cdot \sqrt{K\inparen{\beta_{K+1}^2 \cdot \frac{\kappa}{\alpha_1} +H}} \cdot\sqrt{\sum_{k=1}^K \min\inbraces{ \sum_{h=1}^H\norm{(\hat{V}_k + \lambda I)^{-1/2} \Phi_{kh} C_{kh} \Phi_{kh}^\top (\hat{V}_k + \lambda I)^{-1/2}}, 1}}\\
\le \ &H^{3/2} \cdot \sqrt{K\inparen{\beta_{K+1}^2 \cdot \frac{\kappa}{\alpha_1} +H}}  \cdot \sqrt{\log\det\inparen{\frac{\hat{V}_{K+1}}{\lambda}+I}},
\end{align*}
where the last inequality uses \Cref{lem:logdet}.

From here, we bound $\gamma_{K+1} := \log\det(\hat{V}_{K+1}/\lambda+I)$ via the trace-determinant inequality:
\begin{align*}
&\det\inparen{\frac{1}{\lambda}\hat{V}_{K+1}+I}\leq \inparen{\frac{1}{d_\psi d_\phi}\Tr\inparen{\frac{1}{\lambda}\hat{V}_{K+1}+I}}^{d_\psi d_\phi} \\
&\leq\inparen{\frac{\alpha_2}{d_\psi d_\phi \lambda}\sum_{k=1}^K\sum_{h=1}^H d_\psi\cdot \norm{\phi_{kh}}^2+1}^{d_\psi d_\phi}\leq\inparen{\frac{\alpha_2}{d_\phi \lambda}T B_{\phi}^2+1}^{d_\psi d_\phi}.
\end{align*}

Putting it all together, we get a regret guarantee of:
\begin{align*}
\mathcal{R}(K) &\le H^{3/2} \cdot \sqrt{K\inparen{\beta_{K+1}^2 \cdot \frac{\kappa}{\alpha_1} +H}}  \cdot  \sqrt{\log\det\inparen{\frac{\hat{V}_{K+1}}{\lambda}+I}} + 2H\sqrt{2T \log\inparen{\frac{2}{\delta}}}\\
&\le C \sqrt{\gamma_{K+1} \cdot \inparen{\frac{2\kappa (B_\psi + B_c)}{\alpha_1^3}\inparen{\gamma_{K+1} + \log \nfrac{2}{\delta}}  + \frac{\kappa}{\alpha_1} + H}} \cdot \sqrt{H^2 T} + 2H \sqrt{2T \log \nfrac{2}{\delta}}\\
&= \tilde{O}(d_\psi d_\phi \sqrt{H^3 T}),
\end{align*}
where we use the definition of $\beta_{K+1}$ as well as the bound on the information gain quantity $\gamma_{K+1}$. This concludes the proof.
\end{proof}

%% file: apdx-comparison.tex
\section{Details of comparison to prior work}\label{sec:compare-apdx}

\subsection{Comparison with Exp-UCRL}\label{sec:compare-expucrl}
In this section, we provide further details about our comparison to \citeauthor{chowdhury21} and \citeauthor{kakade2020information}. First, we translate the quantities in \citeauthor{chowdhury21} to our notation:
\version{}{\vspace{-\topsep}}
\version{\begin{itemize}}{\begin{itemize}[leftmargin=0.5cm]}
    \item $\mathbb{A}_{ij} = \Tr(A_i A_j^\top)$; since the $A_i$ take the form of $E_{(k,l)}$, $\mathbb{A} = I_{d_\psi d_\phi}$.
    \item $G_{sa} = \Phi_{sa} \Phi_{sa}^\top$.
    \item $\norm{\theta^\star}_{\mathbb{A}} = \norm{W_0}_F$, so therefore $B_{\mathbb{A}} = B_\star$ in our notation.
    \item $\norm{\mathbb{A}^{-1} G_{sa}} = \norm{G_{sa}} = \norm{\Phi_{sa} \Phi_{sa}^\top}$, so therefore $B_{\varphi, \mathbb{A}} = B_\phi^2$ in our notation.
    \item They require that for all $\theta, s, a$: $\alpha \preceq \mathbb{C}_{sa}^\theta[\psi(s')]\preceq \beta$. Recall that $\mathbb{C}_{sa}^\theta[\psi(s')]$ is the covariance of $\psi(s')$, as introduced in \Cref{assume:bounds} (iv). Therefore, $\beta = \kappa$ in our notation, and we do not require $\alpha$ to be bounded for SMRL. As they state in their paper, $\alpha$ controls the strict convexity of the log partition function that guarantees minimality of the exponential family; it may be possible to remove the dependence on $\alpha$ in their proofs.
    \item We use $K$ for episodes instead of $T$ and $T$ for total number of interactions instead of $N$.
\end{itemize}

Thus, we can restate their theorem (in our notation) as:
\begin{theorem}[Thm.~2 of \cite{chowdhury21}] With regularizer $\lambda > 0$ and any $\delta \in (0,1)$, with probability at least $1-\delta$ the regret of Exp-UCRL can be bounded as:
\[\calR(K) \le \tilde{O} \inparen{\sqrt{\frac{\kappa}{\alpha} \inparen{1 + \frac{\kappa B_\phi^2 H}{\lambda}} \inparen{\lambda B_\star^2 + d_\phi d_\psi} d_\phi d_\psi} \cdot \sqrt{H^2 T}}.\]
\end{theorem}
In comparison, our \Cref{thm:smrl-regret} can be stated as:

\begin{theorem}[\Cref{thm:smrl-regret}, restated] With regularizer $\lambda > 0$ and any $\delta \in (0,1)$, with probability at least $1-\delta$ the regret of SMRL can be bounded as:
    \[\calR(K) \le \tilde{O} \inparen{\sqrt{\inparen{\frac{\kappa(B_\psi + B_c)}{\alpha_1^3} d_\phi d_\psi + \lambda B_\star^2 \frac{\kappa}{\alpha_1} + H} d_\phi d_\psi} \cdot \sqrt{H^2 T}}.\]
\end{theorem}

\subsection{Comparison details for nonLDS}\label{sec:compare-nonlds}

First, we formally state the equivalence between score matching and MLE for the nonLDS problem.

\begin{proposition}\label{prop:equiv}
    Let $\P_W(s'|s,a) := (2\pi\sigma^2)^{-d/2} \exp\inparen{-\frac{1}{2\sigma^2}\norm{s' - W\phi(s,a)}_2^2}$. For any dataset $\calD = \{(s_t, a_t, s_t')\}_{t\in[n]}$:
    \begin{align*}
        \normalfont{\text{(\ref{eq:esml})}} &=\frac{1}{8\sigma^4} \sum_{t=1}^n \norm{s_t' - W\phi(s_t,a_t)}_2^2. \\
        \normalfont{\text{(MLE)}} &:= \argmin_{W} -\sum_{t=1}^n \log \P_W(s_t'|s_t, a_t) = \argmin_{W} \frac{1}{2\sigma^2}\sum_{t=1}^n \norm{s_t' - W\phi(s_t, a_t)}_2^2.
    \end{align*}
\end{proposition}

\Cref{prop:equiv} is straightforward from the definition of the estimators as well as the form of the gaussian density; we omit the proof.

Next, we select the regularization parameter $\lambda$ for each algorithm in order for the estimation procedure to be exactly the same for a fixed dataset.
\version{}{\vspace{-\topsep}}
\version{\begin{itemize}}{\begin{itemize}[leftmargin=0.5cm]}
    \item LC$^3$: $\hat{W} = \argmin_{W} \sum_{t=1}^n \norm{s_t' - W\phi(s_t, a_t)}_2^2 + \tfrac{\lambda_{\text{LC}}}{2} \norm{W}_F^2$ with $\lambda_{\text{LC3}} = \tfrac{\sigma^2}{B_\star^2}$ \cite[Eq.~3.1]{kakade2020information}.
    \item Exp-UCRL: $\hat{W} = \argmin_{W} -\sum_{t=1}^n \log \P_W(s_t'|s_t, a_t) + \tfrac{\lambda_{\text{Exp}}}{2} \norm{W}_F^2$, so we set $\lambda_{\text{Exp}} = \tfrac{1}{2B_\star^2}$ \cite[Eq.~2]{chowdhury21}.
    \item SMRL: $\hat{W} = \argmin_{W} \hat{J}_n(W) + \tfrac{\lambda_\text{SM}}{2} \norm{W}_F^2$, so we set $\lambda_\text{SM} = \tfrac{1}{8\sigma^2 B_\star^2}$ (\Cref{algorithm:smrl}, line 9).
\end{itemize}



%% file: apdx-technical.tex
\section{Technical results}

\begin{lemma}[Concentration of Self Normalized Process]\label{lem:self-normalized}
Let $\{\Delta_t\}_{t=1}^\infty$ be an $\mathbb{R}^m$-valued stochastic process with corresponding filtration $\{\mathcal{F}_t\}_{t=0}^{\infty}$. Let $\Delta_t|\mathcal{F}_{t-1}$ be zero-mean and $\sigma^2$-subgaussian; i.e. $\mathbb{E}[\Delta_t|\mathcal{F}_{t-1}]=0$, and 
$$\forall \lambda\in\mathbb{R}^d,\qquad \mathbb{E}[e^{\lambda^\top\Delta_t}|\mathcal{F}_{t-1}]\leq e^{\sigma^2\|\lambda\|^2/2}.$$
Let $\{\Phi_t\}_{t=1}^{\infty}$ be an $\mathbb{R}^{d\times m}$-valued stochastic process where $\Phi_t\in\mathcal{F}_{t-1}$. Assume $V_0$ is a $d\times d$ positive definite matrix, and let $V_t=\sum_{s=1}^t\Phi_{s}\Phi_{s}^\top$ for $t\geq 1$. Then for any $\delta>0$, with probability at least $1-\delta$, we have for all $t\ge 1$:
$$\left\|\sum_{s=1}^t\Phi_s\Delta_s\right\|_{V_t^{-1}}^2\leq 2\sigma^2\log\frac{\det (V_t+V_0)^{1/2}}{\delta \det V_0^{1/2}}$$
\end{lemma}
\begin{proof}
Let $S_t=\sum_{s=1}^t\Phi_t\Delta_t$, define
$$M_n^\gamma:=\exp\left\{\frac{1}{\sigma}\gamma^\top S_n-\frac{1}{2}\gamma^\top V_n\gamma\right\}.$$
We first prove that $M_n^\gamma$ is a super-martingale adapted to filtration $\{\mathcal{F}_{t-1}\}_{t=1}^{\infty}.$ Define $$D_t^\gamma:=\exp\left\{\frac{1}{\sigma}\gamma^\top\Phi_t\Delta_t-\frac{1}{2}\gamma^\top\Phi_t\Phi_t^\top\gamma\right\}.$$
We have that $M_n^\gamma=M_{n-1}^\gamma D_n^\gamma$. So it suffices to show $\mathbb{E}[D_t^\gamma|\mathcal{F}_{t-1}]\leq 1$. Using the subgaussian property, we have
\[\mathbb{E}[D_t^\gamma|\mathcal{F}_{t-1}]\leq \exp\left\{\sigma^2\frac{1}{2\sigma^2}\|\gamma^\top\Phi_t\|^2-\frac{1}{2}\gamma^\top\Phi_t\Phi_t^\top\gamma\right\} = 1.\]
Next, we use the method of mixtures \cite[\eg][Ch.~20]{lattimore2020bandit}. Define: 
\[\bar M_n:=\int M_n^\gamma f(\gamma)d\gamma,\]
where $f(\gamma)$ is the pdf of normal distribution $\mathcal{N}(0,V_0^{-1})$. On the one hand, we can prove that $\bar M_n$ is a super-martingale adapted to filtration $\{\mathcal{F}_{t-1}\}_{t=1}^\infty$ since
\[\mathbb{E}[\bar M_n|\mathcal{F}_{n-1}]=\mathbb{E}[\mathbb{E}[M_n^\gamma]|\mathcal{F}_{n-1}]=\mathbb{E}[\mathbb{E}[M_n^\gamma|\mathcal{F}_{n-1}]]\leq \mathbb{E}[M_{n-1}^\gamma]=\bar{M}_{n-1}.\]
This implies that the maximal inequality holds for any $\delta\in(0,1)$ \cite[Thm.~3.9]{lattimore2020bandit}:
\begin{equation}
\mathbb{P}\left(\sup_{t\in\mathbb{N}}\bar M_t\geq \frac{1}{\delta}\right)\leq \delta
\label{equation: max inequality}
\end{equation}

On the other hand, we can compute $\bar M_n$ directly. We have 
\begin{align}
\nonumber \bar M_n&=\int\exp\left\{\frac{1}{\sigma}\gamma^\top S_n-\frac{1}{2}\gamma^\top V_n\gamma\right\}f(\gamma)d\gamma\\ 
\nonumber &=\frac{1}{(2\pi)^{m/2}\det V_0^{-1/2}}\int\exp\left\{\frac{1}{\sigma}\gamma^\top S_n-\frac{1}{2}\gamma^\top V_n\gamma-\frac{1}{2}\gamma^\top V_0\gamma\right\}d\gamma\\ 
\nonumber &=\frac{\exp\{\frac{1}{2\sigma^2}S_n^\top(V_n+V_0)^{-1}S_n\}}{(2\pi)^{m/2}\det V_0^{-1/2}}\int\exp\left\{-\frac{1}{2}\|\gamma-(V_n+V_0)^{-1}\frac{S_n}{\sigma}\|^2_{V_n+V_0}\right\}d\gamma\\ 
&=\frac{\det V_0^{1/2}}{\det (V_n+V_0)^{1/2}}\exp\left\{\frac{1}{2\sigma^2}S_n^\top(V_n+V_0)^{-1}S_n\right\}.
\label{equation: mixture}
\end{align}
The last equality is due to integration of the Gaussian pdf. Therefore, by combining (\ref{equation: max inequality}) and (\ref{equation: mixture}) we obtain that with probability $1-\delta$, for all $n\in\mathbb{N}:$
$$\|S_n\|_{(V_n+V_0)^{-1}}^2\leq 2\sigma^2\log\frac{\det (V_n+V_0)^{1/2}}{\delta \det V_0^{1/2}}.$$
\end{proof}

\begin{lemma}[Log Partition Derivatives, \cite{chowdhury21}]\label{lem:logpartitionderivatives}
The first and second derivatives of the exponential family model in \Cref{assume:ef-model} are:
\begin{align*}
    \nabla_{(i,j)} Z_{sa}(W) &= \E_{sa}^{W}[\psi(s')]^\top E_{ij} \phi(s,a), \\
    \nabla_{(i,j), (k,l)}^2 Z_{sa}(W) &= \phi(s,a)^\top E_{ij}^\top \inparen{\E_{sa}^W[\psi(s')\psi(s')^\top] - \E_{sa}^W[\psi(s')]\E_{sa}^W[\psi(s')^\top} E_{kl}\phi(s,a),
\end{align*}
for any $(i,j), (k,l)\in [d_\psi]\times [d_\phi]$.
\end{lemma}

\begin{lemma}\label{lem:sum-subg}
    If $X\in \R^{d}$ is a conditionally $\sigma_1^2$-subgaussian vector and $Y\in \R^{d}$ is a $\sigma_2^2$-subgaussian vector, then $X+Y$ is a conditionally $\sigma_1^2 + \sigma_2^2$-subgaussian vector.
\end{lemma}
    
\begin{proof}
    We have for any $v \in \R^d$ and $p,q \ge 1$ such that $1/p + 1/q = 1$:
    \[\ex{e^{v\cdot (X+Y)} \given \calF_t} \le \ex{e^{p v\cdot X} \given \calF_t}^{1/p} \ex{e^{q v\cdot X} \given \calF_t}^{1/q} \le e^{\norm{v}_2^2 \inparen{p\sigma_1^2 + q\sigma_2^2}/2} = e^{\norm{v}_2^2 (\sigma_1^2 + \sigma_2^2)/2}.\]
    where we use H\"{o}lder's inequality in the first inequality, subgaussianity of $X,Y$ in the second inequality, and set $p = \frac{\sigma_2}{\sigma_1} + 1$ in the equality.
\end{proof}

\begin{lemma}
Suppose function $f:\mathbb{R}^m\rightarrow[0,1]$, and $p,q$ are two probability density functions on $\mathbb{R}^m$. Suppose $s$ is an $\mathbb{R}^m$-valued random variable. Then it holds that 
\[\abs{\mathbb{E}_pf(s)-\mathbb{E}_q f(s)}\leq \mathrm{TV}(p\Vert q).\]
\label{lemma: TV}
\end{lemma}
\begin{proof}
Denote $E=\{x\in\mathbb{R}^m:p(x)>q(x)\}, F=\mathbb{R}^m\backslash E.$ We have
\begin{align*}
\mathbb{E}_pf(s)-\mathbb{E}_qf(s)&=\int_{\mathbb{R}^m} f(x)\left(p(x)-q(x)\right)dx\\
&=\int_{E}f(x)\left(p(x)-q(x)\right)dx-\int_{F}f(x)\left(q(x)-p(x)\right)dx\\
\Rightarrow \abs{\mathbb{E}_pf(s)-\mathbb{E}_qf(s)} &\leq \max\left\{\int_{E}f(x)\left(p(x)-q(x)\right)dx,\int_{F}f(x)\left(q(x)-p(x)\right)dx\right\}\\
&=\max\left\{\left|\int_{E}f(x)\left(p(x)-q(x)\right)dx\right|,\left|\int_{F}f(x)\left(p(x)-q(x)\right)dx\right|\right\}\\
&\leq \sup_{A\subset\mathbb{R}^m}\left|\int_Af(x)(p(x)-q(x))dx\right| = \mathrm{TV} 
(p\|q).
\end{align*}
\end{proof}



%% file: apdx-experiments.tex
\section{Experimental details}\label{apdx:exp_detail}

We explain the experimental setup in greater detail. Our code can be found on the \href{https://github.com/anmolkabra/score-matching-rl}{Github}.

\subsection{MDP construction}
We recall the MDP setup from \Cref{sec:experiment}. The MDP has $\calS = \R$, $\calA = \{+1, -1\}$, and $H=10$. The initial state distribution is $\mathrm{Unif}([-1, +1])$. The transition and reward are given by
\begin{align}\label{eq:transition}
    \P(s'|s,a) &= \exp\inparen{-\frac{s'^{1.7}}{1.7}} \cdot \exp\inparen{\sin(4s') (s+a)}, \\
    r(s,a) &= \exp\inparen{-10\inparen{s - \frac{\pi}{8}}^2} + \exp\inparen{-10\inparen{s + \frac{3\pi}{8}}^2}.
\end{align}

Thus, the transition can be written in the form of \Cref{assume:ef-model} with 
\begin{align*}
    q(s') \coloneqq \exp\inparen{-\frac{s'^{1.7}}{1.7}}, \quad \psi(s') \coloneqq \sin(4s'), \quad W_0 = [1,1], \quad \phi(s,a) = [s,a]^\top.
\end{align*}

The MDP has the property that states will mostly be confined to the region $[-3,+3]$, and the conditional density $\P(s'|s,a=+1)$ generally has higher mass in regions with reward than the conditional density $\P(s'|s,a=-1)$ does. Thus, selecting $a=+1$ is preferred over $a=-1$. (We can get the opposite to happen by setting $W_0 = [1, -1]$ instead.) This is illustrated in \Cref{fig:exp-densities} for various starting states $s$.

\begin{figure}[h]
    \centering
    {\includegraphics[width=0.95\linewidth]{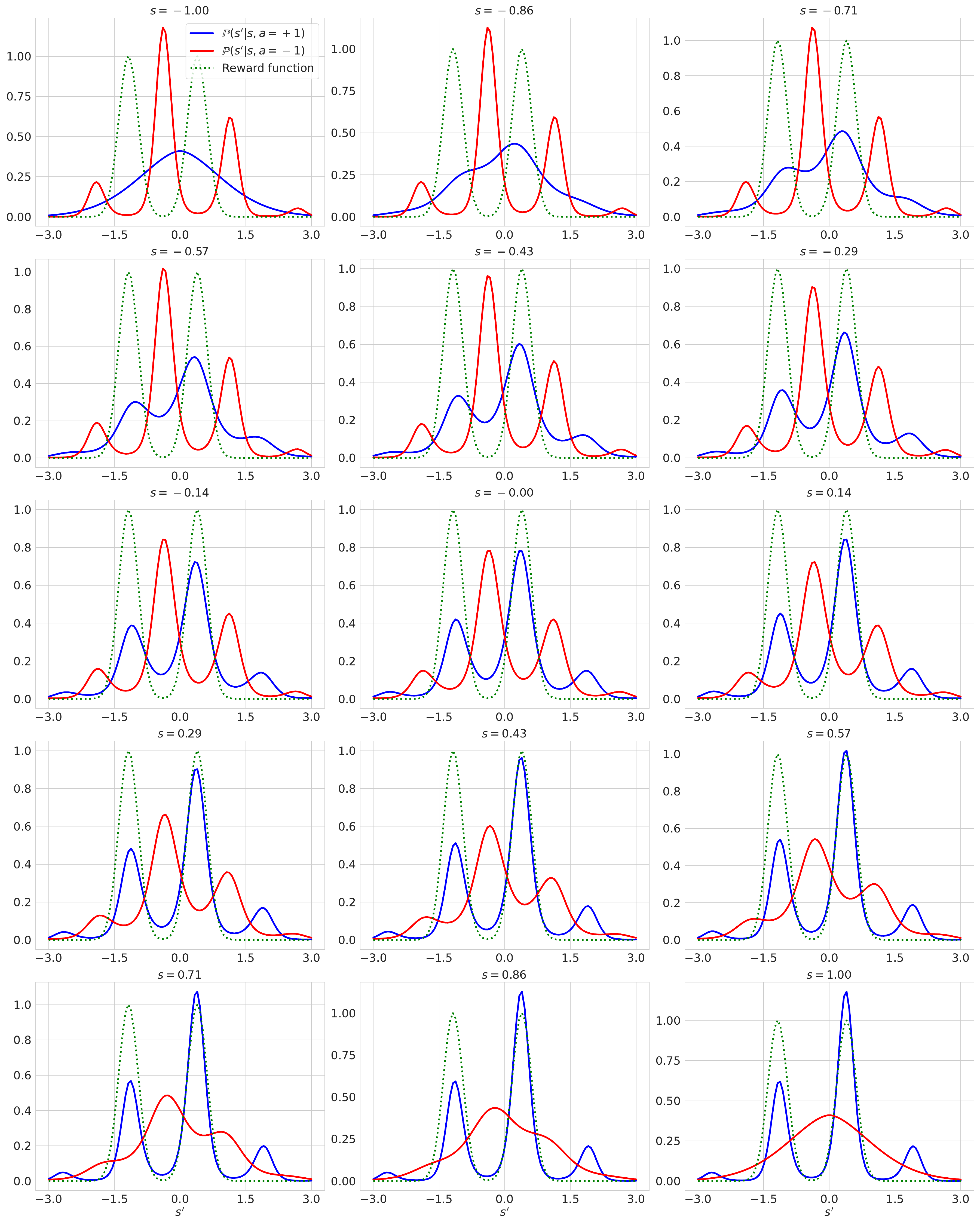}}
    \caption{Plotting $\P(s'|s,a)$ for various start states for $a=+1$ (blue) and $a=-1$ (red); reward is superimposed (dotted green). For the plotted states $s \in [-1, +1]$, taking $a=+1$ is more likely to transition to states with higher reward than taking $a=-1$ will.}
    \label{fig:exp-densities}
    \vspace*{-1.5em}
\end{figure}

\subsection{Model estimation}
We consider two model estimation methods, given a dataset of transitions $\calD = \{(s_t, a_t, s_t')\}_{t\in [n]}$.
\version{}{\vspace{-\topsep}}
\version{\begin{itemize}}{\begin{itemize}[leftmargin=0.5cm]}
    \item We can use the score matching estimator \normalfont{(\ref{eq:esme})} for model class $\calP(\psi, \phi, q)$ to get an estimate $\hat{W}_\mathrm{SM}$ for $W_0$ and corresponding estimated transition model $\hat{\P}_\mathrm{SM}$.
    \item We can treat the system as an LDS and solve the least squares problem
\begin{align*}
    \hat{W}_\mathrm{LS} = \argmin_{W\in \R^2}\sum_{t=1}^n \norm{s'_t - W \phi(s_t, a_t)}_2^2 + \frac{\lambda}{2} \norm{W}_F^2.
\end{align*}
    Then we estimate the transition model $\hat{\P}_\mathrm{LDS}$ as $s' = \langle \hat{W}_\mathrm{LS}, \phi(s,a)\rangle + \eta$, where $\eta \sim \calN(0, \sigma^2)$ is independent noise. This approach is the estimation procedure of LC$^3$ with the feature mapping $\phi(s,a) = [s,a]^\top$.
\end{itemize}
In both methods we set the regularization parameter $\lambda = 0$. 

We also remark that since we are working with a low-dimensional transition function, a third option is the estimation procedure of Exp-UCRL: numerically compute the MLE with the given model class $\calP(\psi, \phi, q)$. Experimentally for the RL task, we expect this to obtain similar performance with score matching, since for this density, MLE and score matching learn very similar models. However, we do not investigate this approach because MLE is only computationally tractable for low-dimensional problems and will not scale to higher-dimensional MDPs.

In \Cref{fig:exp-estimation}, we plot the estimated densities for various states $s$ under i.i.d.~transition data. Since the ground truth model is well-specified in the class $\calP(\psi, \phi, q)$ (and also not representable as an LDS), score matching learns a more accurate transition model than fitting an LDS does.

\begin{figure}[]
    \centering
    {\includegraphics[width=0.95\linewidth]{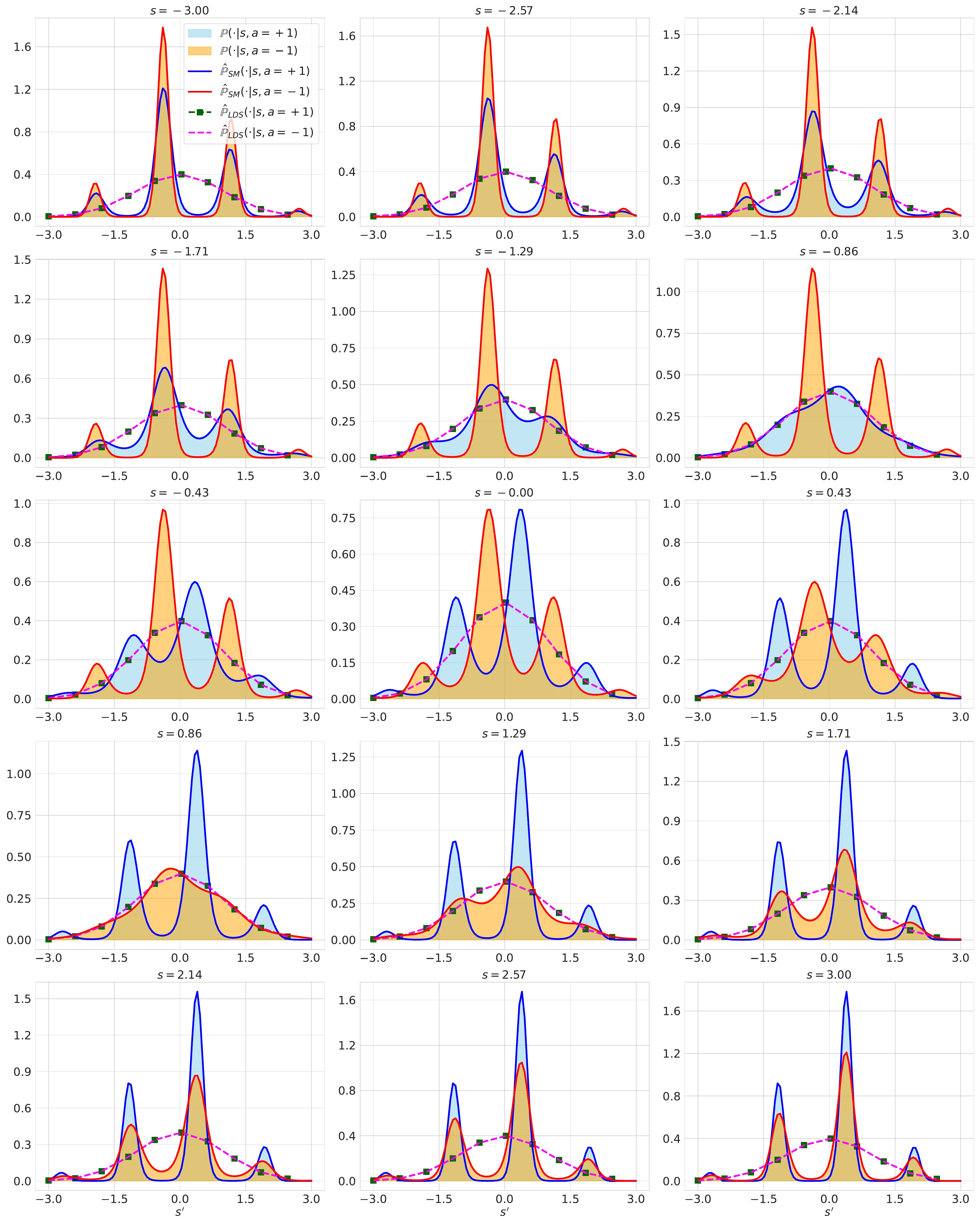}}
    \caption{Plotting learned densities via score matching and fitting an LDS under i.i.d.~data. We generate $n=1000$ samples where $(s,a)$ are sampled i.i.d.~from $\mathrm{Unif}([-1, +1])\times \mathrm{Unif}(\{-1, +1\})$ and the $s'$ are drawn according to \Cref{eq:transition}. The ground truth densities for $a=+1$ and $a=-1$ are plotted as the shaded blue and orange regions respectively. Score matching learns the correct density for both $a=+1$ (blue lines) and $a=-1$ (red lines), while the LDS model is unable to differentiate between the two, and essentially learns the same model $\hat{\P}_\mathrm{LDS}(s'|s, a=+1) \approx \hat{\P}_\mathrm{LDS}(s'|s, a=-1)$ (green and magenta lines).}
    \label{fig:exp-estimation}
    \vspace*{3em}
\end{figure}

\subsection{RL task evaluation}
We evaluate end-to-end performance for the RL task. To isolate the benefit of better model estimation via score matching, we fix a simple random sampling shooting (RSS) planner. (Recall from \Cref{sec:cc} that we cannot implement SMRL and LC$^3$ as written because of the computational intractability of optimistic planning. The RSS planner is an effective, easy-to-implement alternative.) In every call, the RSS planner takes as input an estimated model $\hat{\P}$ and a reward function $r$ and an initial state $s_\mathrm{init}$. To evaluate the action $a$, it simulates a dataset of $n_\mathrm{rollouts}$ independent trajectories, each of length $\tau$, where the initial state in each trajectory is $s_\mathrm{init}$, action $a$ is played in every step, and subsequent states are drawn using the estimated model $\hat{\P}$; the estimated value of taking action $a$ is the average of the cumulative rewards for these trajectories. Lastly, the planner selects the action in the real environment which has highest simulated average reward. In experiments, we set the hyperparameters $n_\mathrm{rollouts} \coloneqq 100$ and $\tau \coloneqq 5$.

To evaluate the estimation methods, we run the RSS planner in the MDP at every step $h \in [H]$. The model is re-estimated at the end of every episode using the collected transition data, using either score matching or fitting an LDS. We also compare against a baseline where the RSS planner is supplied with the ground truth model $\P$.